\documentclass[10pt]{article}
\pdfoutput=1

\usepackage{booktabs} 
\usepackage{url}
\usepackage{amsmath}
\usepackage{amsfonts}
\usepackage{float}
\usepackage{subfig}

\usepackage{graphicx}
 \usepackage{bmpsize}

\usepackage{color}
\usepackage[noend]{algorithmic}
\usepackage{algorithm}
\usepackage{multirow}
\usepackage{verbatim}
\usepackage{amsmath}
\usepackage{amsthm}
\usepackage{indentfirst}

\DeclareMathOperator*{\argmax}{arg\,max}

\begin{document}
%
\title{Geographical Hidden Markov Tree for Flood Extent Mapping}
\date{}
%
%
%
%

\author{Miao Xie, Zhe Jiang\thanks{Contacting Author: Zhe Jiang, zjiang@cs.ua.ede; Xie Miao, Zhe Jiang, and Arpan Man Sainju are all from the Department of Computer Science, University of Alabama, Tuscaloosa}, Arpan Man Sainju}

\maketitle
\begin{abstract}
Flood extent mapping plays a crucial role in disaster management and national water forecasting. Unfortunately, traditional classification methods are often hampered by the existence of noise, obstacles and heterogeneity in spectral features as well as implicit anisotropic spatial dependency across class labels. 
In this paper, we propose geographical hidden Markov tree, a probabilistic graphical model that generalizes the common hidden Markov model from a one dimensional sequence to a two dimensional map. Anisotropic spatial dependency is incorporated in the hidden class layer with a reverse tree structure. We also investigate computational algorithms for reverse tree construction, model parameter learning and class inference. Extensive evaluations on both synthetic and real world datasets show that proposed model outperforms multiple baselines in flood mapping, and our algorithms are scalable on large data sizes.
\end{abstract}



%

\section{Introduction}\label{sec:intro}

Flood extent mapping plays a crucial role in addressing grand societal challenges such as disaster management,  national water forecasting, as well as energy and food security. For example, during Hurricane Harvey floods in 2017, first responders needed to know where flood water was in order to plan rescue efforts. In national water forecasting, detailed flood extent maps can be used to calibrate and validate the NOAA National Water Model~\cite{nwm}, which can forecast the flow of over 2.7 million rivers and streams through the entire continental U.S.~\cite{iwrss}. 

In current practice, flood extent maps are mostly generated by flood forecasting models, whose accuracy is often unsatisfactory in high spatial details~\cite{iwrss}. Other ways to generate flood maps involve sending field crew on the ground to record high-water marks, or visually interpreting earth observation imagery~\cite{brivio2002integration}. However, the process is both expensive and time consuming. With the large amount of high-resolution earth imagery being collected from satellites (e.g., DigitalGlobe, Planet Labs), aerial planes (e.g., NOAA National Geodetic Survey), and unmanned aerial vehicles, the costs of manually labeling flood extents become prohibitive.

The focus of this paper is to develop a classification model that can automatically classify earth observation imagery pixels into flood extent maps. The results can be used by first responders to plan rescue efforts, by hydrologists to calibrate and validate water forecasting models, as well as by insurance companies to process claims. Specifically, we can utilize a small set of manually collected ground truth (flood and dry locations) in one earth imagery to learn a classification model. Then the model can be used to classify flood pixels in other imagery where ground truth is not available.

However, flood mapping poses several unique challenges that are not well addressed in traditional classification problems. First, data contains rich noise and obstacles. For example, high-resolution earth imagery often has noise, clouds and shadows. The spectral features of these pixels cannot be used to distinguish classes. Second, class confusion exists due to heterogeneous features. For instance, pixels of tree canopies overlaying flood water have the same spectral features with those trees in dry areas, yet their classes are different. Third, implicit directed spatial dependency exists between flood class locations. Specifically, due to gravity, flood water tends to flow to nearby lower locations following topography. Such dependency is not uniform in all directions (anisotropic). Finally, the data volume is huge in high-resolution imagery (e.g., hundreds of millions of pixels in one city), requiring scalable algorithms.

To address these challenges, we propose a novel spatial classification model called \emph{geographical hidden Markov tree (HMT)}. It is a probablistic graphical model that generalizes the common hidden Markov model (HMM) from a one-dimensional sequence to a two dimensional geographical map. Specifically, the hidden class layer contains nodes (pixels) in a reverse tree structure to represent anisotropic spatial dependency with a partial order constraint. Each hidden class node has an associated observed feature node for the same pixel. Such a unique model structure can potentially reduce classification errors due to noise, obstacles, and heterogeneity among spectral features of individual pixels. 

We further investigate computational algorithms for reverse tree construction, model parameter learning, and class inference. Specifically, reverse tree is constructed following topological orders based on elevations. In order to learn model parameters given a hidden class layer, we utilize the EM algorithm with message propagation along the reverse tree. Finally, for class inference, we design a greedy algorithm that assign class labels for tree nodes to maximize overall probability following the partial order constraint.

In summary, we make the following contributions:
\begin{itemize}
    \item We propose a novel geographical hidden Markov tree (HMT) model that incorporates partial order class dependency in a reverse tree structure in a hidden class layer. Unlike existing hidden Markov trees~\cite{crouse1998wavelet} which model dependency in two-dimensional time-frequency domain for signal processing, our geographical HMT captures anisotropic (directed) spatial dependency with a partial order constraint.
    \item We design efficient algorithms for reverse tree construction, model parameter learning and class inference.
    \item We conduct theoretical analysis on the correctness and time complexity of HMT algorithms.
    \item We evaluate proposed model in both synthetic and real world datasets for flood mapping. Results show that proposed model outperforms multiple baseline methods in flood mapping, and our algorithms are scalable for large data sizes.
\end{itemize}


\section{Problem Statement}\label{sec:prob}
\subsection{Preliminaries}

A \emph{spatial raster framework} is a tessellation of a two dimensional plane into a regular grid of $N$ cells. Spatial neighborhood relationship exists between cells based on cell adjacency. The framework can consist of $m$ non-spatial explanatory feature layers (e.g., spectral bands in earth imagery), one spatial contextual layer (e.g., elevation), and one class layer (e.g., \emph{flood}, \emph{dry}).

Each cell in a raster framework is a \emph{spatial data sample}, noted as $\mathbf{s_n}=(\mathbf{x}_n, \phi_n, y_n)$, where $n\in \mathbb{N}, 1\leq n \leq N$, $\mathbf{x}_n\in \mathbb{R}^{m\times 1}$ is a vector of non-spatial explanatory feature values with each element corresponding to one feature layer, $\phi_n\in\mathbb{R}$ is a cell's value in the spatial contextual layer, and $y_n\in \{0,1\}$ is a binary class label.

A raster framework with all samples is noted as $\mathcal{F}=\{\mathbf{s_n}|n\in \mathbb{N}, 1\leq n \leq N\}$, non-spatial explanatory features of all samples  are noted as $\mathbf{X}=[\mathbf{x}_1,...,\mathbf{x}_N]^T$, the spatial contextual layer is noted as $\boldsymbol{\Phi}=[\phi_1,...,\phi_N]^T$, and the class layer is noted as $\mathbf{Y}=[y_1,...,y_N]^T$.

Due to physics, spatial dependency exists between cells based on their values in the spatial contextual layer. Such dependency is often non-uniform in different directions (\emph{anisotropic}). For example, due to gravity, flood water can only flow to neighboring cells with lower elevation values.

Anisotropic dependency often follows a \emph{partial order constraint}. Formally, assuming the spatial contextual layer is a potential field (e.g., elevation), a partial order dependency $\mathbf{s}_i\leadsto\mathbf{s}_j$ exists if and only if there exist a sequence of neighboring (adjacent) cells $<\mathbf{s}_i,\mathbf{s}_{p_1},\mathbf{s}_{p_2},...,\mathbf{s}_{p_l},\mathbf{s}_j>$ such that $\phi_j\geq\phi_i$ and $\phi_j\geq\phi_{p_k}$ for any $1\leq k \leq l$.

Figure~\ref{fig:tree}(a) shows an illustrative example with eight spatially adjacent cell samples in one dimensional space. Due to gravity, if cell $\mathbf{s}_5$ is \emph{flood}, its nearby cells with lower elevations including $\mathbf{s}_2,\mathbf{s}_3,\mathbf{s}_4,\mathbf{s}_6,\mathbf{s}_7$ should also be \emph{flood}, even if their feature values indicate otherwise. Thus, we can establish partial order spatial dependency between cell locations such as $\mathbf{s}_4\leadsto\mathbf{s}_5$, $\mathbf{s}_2\leadsto\mathbf{s}_5$.

\begin{figure}
    \centering
    \subfloat[Eight consecutive sample locations in one dimensional space]{\includegraphics[width=1.5in]{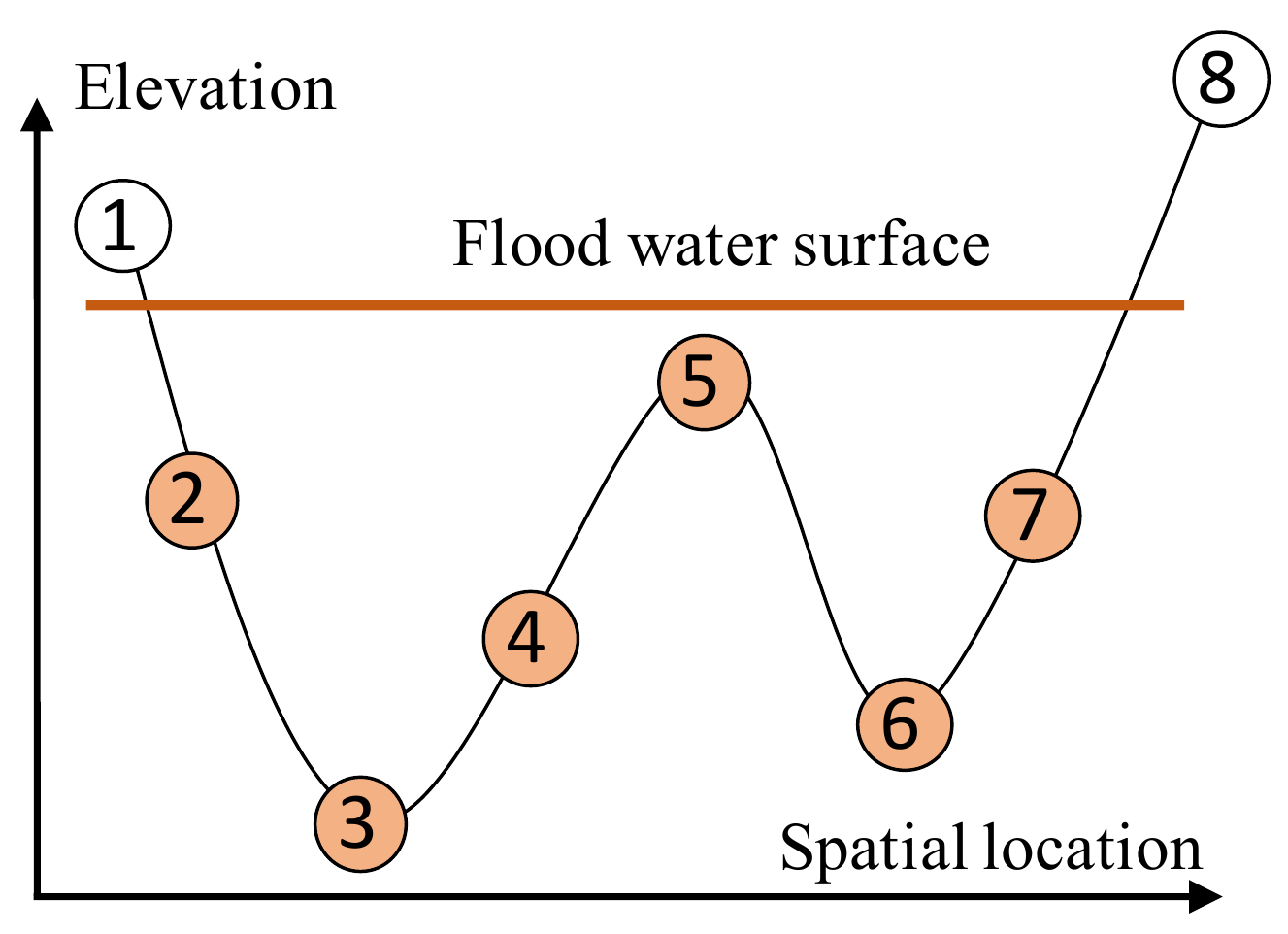}}\hspace{1mm}
    \subfloat[Partial order constraint in a reverse tree]{\includegraphics[width=1in]{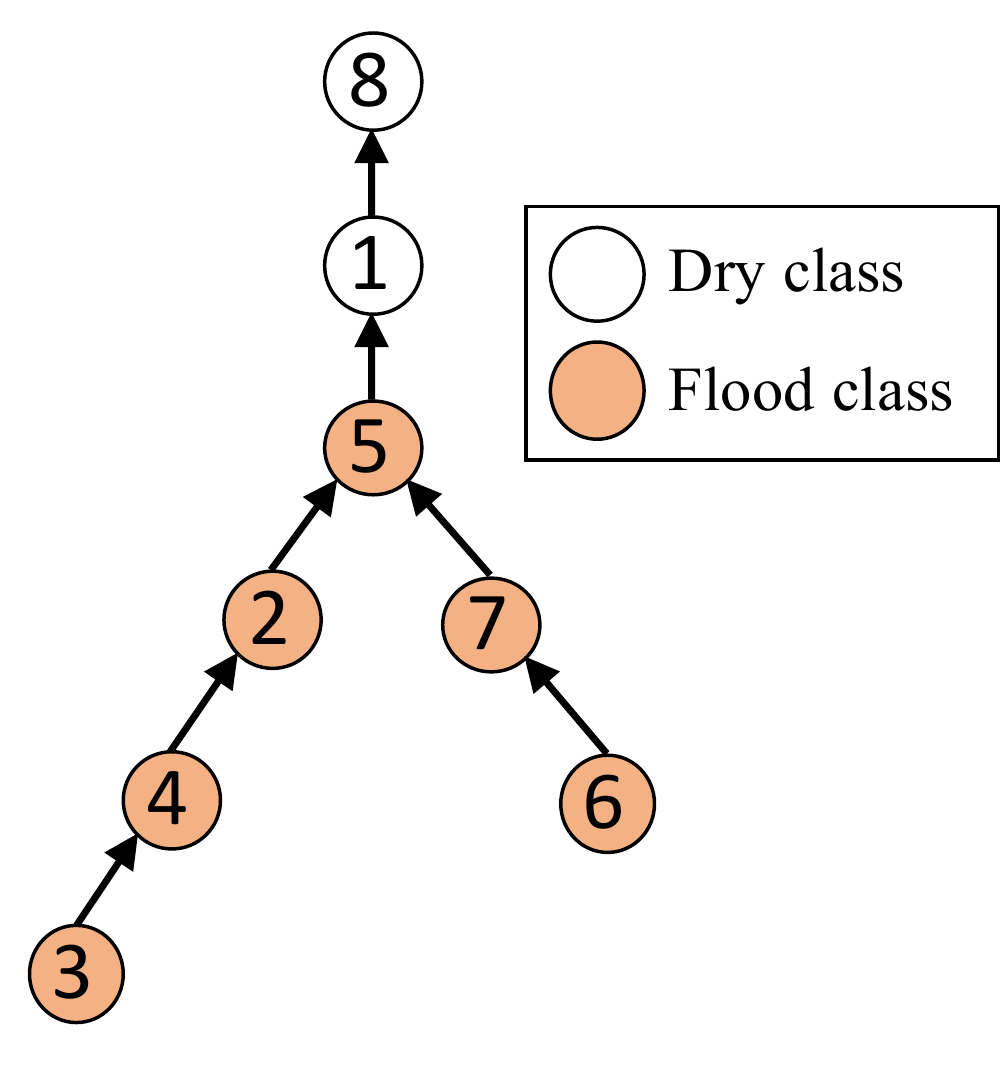}}\hspace{1mm}
    \caption{Illustration of partial order class dependency}
    \label{fig:tree}
\end{figure}

Partial order dependency across all pairs of samples in a raster framework can be represented by a reverse tree structure, which is called \emph{spatial dependency (reverse) tree} or \emph{dependency tree}. We sometimes omit the word ``reverse" for simplicity. The tree structure removed some redundant dependency between cell locations. Due to the reverse nature, a tree node $n$ can have at most one child  $C_n\in \mathbb{N}$, but multiple parents $\mathcal{P}_n=\{k\in \mathbb{N}|\mathbf{s}_k\rightarrow\mathbf{s}_n\}$ and multiple siblings $\mathcal{S}_n=\{k\in \mathbb{N}|\exists~c\in \mathbb{N}~s.t.~ \mathbf{s}_k\rightarrow\mathbf{s}_c, \mathbf{s}_n\rightarrow\mathbf{s}_c\}$, where $\rightarrow$ represents a tree edge from a parent to a child.  

Figure~\ref{fig:tree}(b) shows an example of dependency tree corresponding to samples in Figure~\ref{fig:tree}(a). Class dependency $\mathbf{s}_3\leadsto\mathbf{s}_2$ is redundant given dependency $\mathbf{s}_3\rightarrow\mathbf{s}_4$ and $\mathbf{s}_4\rightarrow\mathbf{s}_2$. It is worth noting that we assume an arbitrary order when comparing nodes with the same elevation values. For instance, if node $\mathbf{s}_1$ and node $\mathbf{s}_8$ had the same elevation, the top of the tree could be either $\mathbf{s}_5\rightarrow\mathbf{s}_1\rightarrow\mathbf{s}_8$ or $\mathbf{s}_5\rightarrow\mathbf{s}_8\rightarrow\mathbf{s}_1$.


\subsection{Formal problem definition}
We now formally define the problem.

{\noindent}{\bf Input:}\\
$\bullet$ Spatial raster framework $\mathcal{F}=\{\mathbf{s}_n|n\in \mathbb{N}, 1\leq n \leq N\}$\\
$\bullet$ Explanatory features of samples $\mathbf{X}=[\mathbf{x}_1,...,\mathbf{x}_N]^T$\\
$\bullet$ Spatial contextual layer (elevation) of samples: $\boldsymbol{\Phi}=[\phi_1,...,\phi_N]^T$\\
$\bullet$ Training samples $\{\mathbf{s}_k|k\in training~set\}$ \\
{\bf Output:} A spatial classification model $f:\mathbf{Y}=f(\mathbf{X})$\\
{\bf Objective:} minimize classification errors\\
{\bf Constraint:}\\
$\bullet$ Explanatory feature layers contain noise and obstacles\\
$\bullet$ Partial order dependency exists between sample classes based on spatial contextual layer\\
$\bullet$ Sample class is binary, $y_n\in\{0,1\}$

\section{Proposed Approach}\label{sec:app}
In this section, we start with overview of our hidden Markov tree model and its probabilistic formulation. We then introduce specific algorithms for dependency tree construction, model parameter learning and class inference.

\subsection{Overview of Hidden Markov Tree}
We propose a hidden Markov tree (HMT) model, which generalizes the common hidden Markov model from a total order chain structure to a partial order (reverse) tree structure. As illustrated in Figure~\ref{fig:hmtrt}, a HMT model consists of two layers: a hidden layer of sample classes (e.g., flood, dry), and an observation layer of sample feature vectors (e.g., spectral vectors). Each node corresponds to a spatial data sample (raster cell). Edge directions show probabilistic conditional dependence structure. Specifically, the model assumes that feature vectors of different samples are conditionally independent with each other given their classes, and sample classes follow a partial order dependency in a reverse tree structure.

\begin{figure}[h]
    \centering
    \includegraphics[width=1.5in]{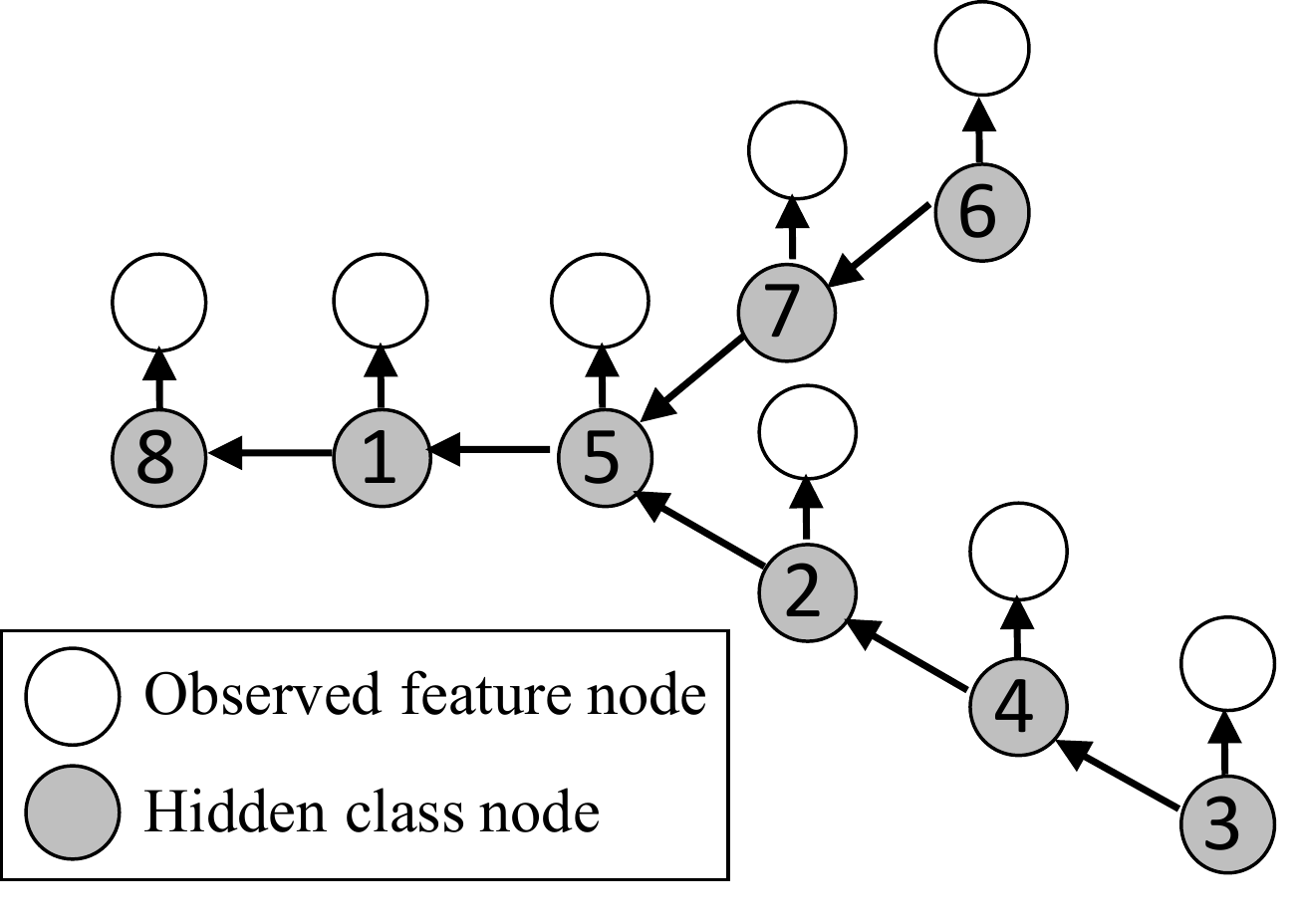}
    \caption{Illustration of hidden Markov tree framework}
    \label{fig:hmtrt}
\end{figure}
Hidden Markov tree is a probabilistic graphic model. The joint distribution of all samples' features and classes can be expressed as Equation~\ref{eq:joint}, where ${\mathcal{P}_n}$ is the set of parent samples of the $n$th sample in the dependency tree, and $y_{k\in\mathcal{P}_n}\equiv\{y_k|k\in\mathcal{P}_n\}$ is the set of class nodes corresponding to parents of the $n$th sample. For a leaf node $n$, ${\mathcal{P}_n}=\emptyset$, and $P(y_n|y_{k\in\mathcal{P}_n})=P(y_n)$.
\begin{equation}\label{eq:joint}
    P(\mathbf{X},\mathbf{Y})= P(\mathbf{X}|\mathbf{Y})P(\mathbf{Y}) = \prod_{n=1}^N P(\mathbf{x}_n|y_n) \prod_{n=1}^NP(y_n|y_{k\in\mathcal{P}_n})
\end{equation}

The conditional probability of sample feature vector given its class can be assumed i.i.d. Gaussian for simplicity, as shown in Equation~\ref{eq:featureclassprob}, where $\boldsymbol{\mu}_{y_n}$ and $\boldsymbol{\Sigma}_{y_n}$ are the mean and covariance matrix of feature vector $\mathbf{x}_n$ for class $y_n$ ($y_n=0,1$). It is worth noting that $P(\mathbf{x}_n|y_n)$ could be more general than i.i.d. Gaussian.
\begin{equation}\label{eq:featureclassprob}
    P(\mathbf{x}_n|y_n)\sim \mathcal{N}(\boldsymbol{\mu}_{y_n},\boldsymbol{\Sigma}_{y_n})
\end{equation}

Class transitional probability follows the partial order constraint. For example, due to gravity, if any parent's class is \emph{dry}, the child's class must be \emph{dry}; if all parents' classes are \emph{flood}, then the child has a high probability of being \emph{flood}. Consider \emph{flood} as the positive class (class value $1$) and \emph{dry} as the negative class (class value $0$), the transitional probability is actually conditioned on the product of parent classes $y_{\mathcal{P}_n}\equiv\prod_{k\in\mathcal{P}_n}y_k$. The formula is in Equation~\ref{eq:clscondition}, where $\rho$ is the probability of a child in class $1$ given all parents in class $1$ (note that we assume $0^0\equiv1$). In other words, if any parent is in class $0$ ($y_{\mathcal{P}_n}=0$), the current node must also be in class $0$ ($y_n=0$); if all parents are in class $1$ ($y_{\mathcal{P}_n}=1$), then the current node has a probability of $\rho$ being in class $1$.
\begin{equation}\label{eq:clscondition}
    \small
    P(y_n|y_{\mathcal{P}_n})=1^{(1-y_n)(1-y_{\mathcal{P}_n})}\times 0^{y_n(1-y_{\mathcal{P}_n})} \times \rho^{y_ny_{\mathcal{P}_n}} \times (1-\rho)^{(1-y_n)y_{\mathcal{P}_n}}
\end{equation}

For a leaf node $n$, ${\mathcal{P}_n}=\emptyset$. The transitional probability is degraded into simple class probability $P(y_n|y_{k\in\mathcal{P}_n})\equiv P(y_n)=\pi^{y_n}\times(1-\pi)^{1-y_n}$, where $\pi$ is the probability of $y_n$ being in class $1$.

Though we introduce our HMT in the context of flood mapping, the model can potentially be used for a broad class of spatial classification problems in which class labels follow a partial order dependency. Examples include predicting pollutants in river stream networks and traffic congestion in road networks. 

\subsection{Dependency Tree Construction}
Given geopotential field values (e.g., elevation) of all cells in a raster framework, the goal is to produce a partial order class dependency tree, in which each node corresponds to the class label of a cell. The process is computationally challenging due to the large number of cells (tree nodes) in real world high-resolution earth imagery (e.g., hundreds of millions of pixels). 
\begin{algorithm}
\caption{{Spatial Dependency Tree Construction}}
\label{alg:treebaseline}
\begin{algorithmic}[1]
\REQUIRE\quad\\
$\bullet$ A raster framework of samples: $\mathcal{F}=\{\mathbf{s_n}|n\in \mathbb{N}, 1\leq n \leq N\}$\\
$\bullet$ A spatial contextual layer of samples: $\boldsymbol{\Phi}=[\phi_1,...,\phi_N]^T$\\
\ENSURE\quad\\
$\bullet$ A spatial dependency tree
\STATE Initialize all samples as \emph{unvisited}
\STATE Sort all samples by increasing $\phi$ values 
\FORALL {sample $\mathbf{s_n}$ in an ascending order of $\phi$}
    \STATE Mark $\mathbf{s_n}$ as \emph{visited}
    \STATE Create a new tree node of $\mathbf{s_n}$
    \IF {there exists \emph{unvisited} neighbor of $\mathbf{s_n}$}
        \FORALL {\emph{unvisited} neighbor $\mathbf{s_k}$ of $\mathbf{s_n}$}
            \STATE Traverse from node $\mathbf{s_k}$ to the rear of its tree branch
            \STATE Attach node $\mathbf{s_n}$ to the rear if not have done so
        \ENDFOR
    \ELSE
        \STATE Create a tree branch starting from node $\mathbf{s_n}$ as a leaf
    \ENDIF
\ENDFOR
\STATE {\bf return} the root node of dependency tree
\end{algorithmic}
\end{algorithm}

To address the challenge, we propose an algorithm that constructs the tree by adding nodes in topological order. Details are in Algorithm~\ref{alg:treebaseline}. The algorithm starts with an empty tree and an empty set of \emph{visited} cells (all cells are \emph{unvisited}, step 1). It sorts all cells by their geopotential field (elevation) values (step 2). After this, \emph{unvisited} cells are added into the tree (i.e., become visited) one by one following an ascending order of geopotential. Specifically, for each cell, the algorithm first marks it as \emph{visited} (step 4), creates a tree node for the cell (step 5), and attach the tree node to the rear of the tree branch following every \emph{visited} neighbor of the the cell (steps 6 to 9). If no neighbor of the cell is \emph{visited}, the cell is a local minimum in geopotential field, and the algorithm creates a new tree branch starting from the node of the cell (steps 10 to 11). 

We now use the example of Figure~\ref{fig:tree} to illustrate the algorithm execution trace. The example contains cells in one dimensional space, but generalization to the case of two dimensional space is trivial. The input contains eight cells from $\mathbf{s}_1$ to $\mathbf{s}_8$. The algorithm first sorts these cells by ascending order of elevation, and gets a sequence $\mathbf{s}_3,\mathbf{s}_6,\mathbf{s}_4,\mathbf{s}_7,\mathbf{s}_2,\mathbf{s}_5,\mathbf{s}_1,\mathbf{s}_8$. Then, leaf nodes are created for $\mathbf{s}_3$ and $\mathbf{s}_6$, since none of their neighbors are visited by then. Next, when adding $\mathbf{s}_4$, its neighbor $\mathbf{s}_3$ is \emph{visited}, so the algorithm adds node $\mathbf{s}_4$ to the rear of the branch following $\mathbf{s}_3$. Similarly, node $\mathbf{s}_7$ and $\mathbf{s}_2$ are attached to the two branches respectively. When adding the node for $\mathbf{s}_5$, both of its neighbors are \emph{visited}, so $\mathbf{s}_5$ is attached to the rear of both branches. After this, nodes $\mathbf{s}_1$ and $\mathbf{s}_8$ are added consecutively.

Time complexity analysis: Algorithm~\ref{alg:treebaseline} involves a one-time sorting of $N$ cells, which is $O(N\log N)$. Then, for each of the $N$ cells, the main operation is to attach the cell to the rear of the branches of its visited neighbors. A naive implementation will cost $O(N)$, making the total cost $O(N^2)$. A smarter way to do this is to maintain a rear node pointer for each branch when it is created (i.e., when a leaf node is added). Assuming that geopotential field values on neighboring cells are contiguous (this is often true since real world elevation of nearby locations do not change suddenly), finding the rear of a neighboring cell's branch is within a constant cost, making the total time cost $O(N\log N + N)=O(N\log N)$ (cost after sorting is linear).

\begin{figure}
\centering
\subfloat[From leaves to root]{%
      \includegraphics[width=1.5in]{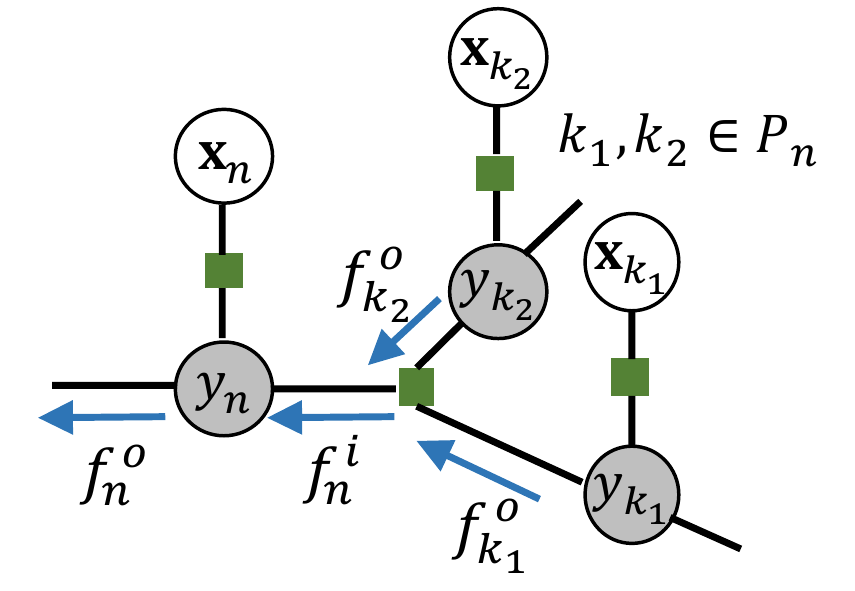}
}
\subfloat[From root to leaves]{%
      \includegraphics[width=1.3in]{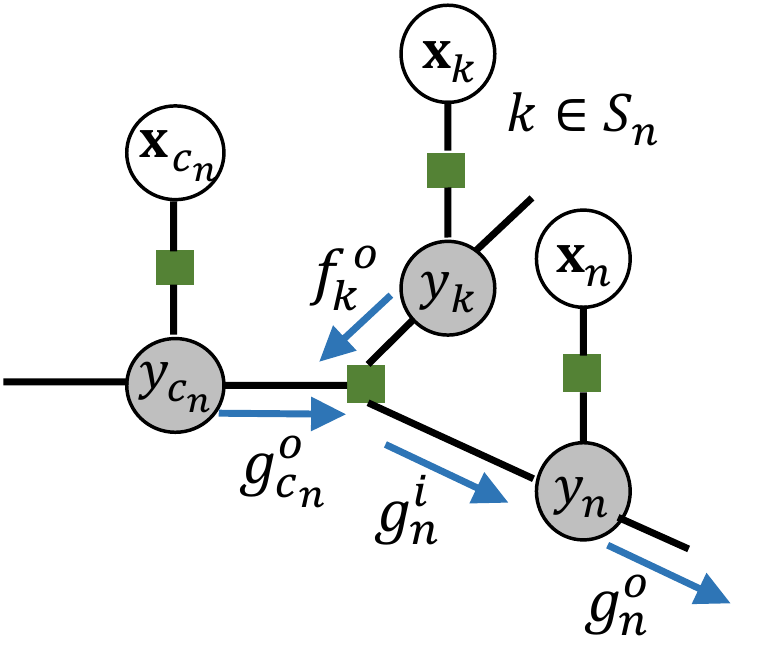}
}
\caption{Illustration of message propagation in a HMT}
\label{fig:message}
\end{figure}

\subsection{Model Parameter Learning}
The parameters of hidden Markov tree include the mean and covariance matrix of sample features in each class, prior probability of leaf node classes, and class transition probability for non-leaf nodes. We denote the entire set of parameters as $\boldsymbol{\Theta}=\{\rho, \pi, \boldsymbol{\mu}_c, \boldsymbol{\Sigma}_c|c=0,1 \}$. Learning the set of parameters poses two major challenges: first, there exist unknown hidden class variables $\mathbf{Y}=[y_1,...,y_N]^T$, which are non-i.i.d.; second, the number of samples (nodes) is huge (up to hundreds of millions of pixels). 

To address these challenges, we propose to use the expectation-maximization (EM) algorithm and message (belief) propagation. Our EM-based approach has the following major steps:
\renewcommand{\labelenumi}{(\alph{enumi})}
\begin{enumerate}
    \item Initialize parameter set $\boldsymbol{\Theta}_0$
    \item Compute posterior distribution of hidden classes:\\ $P(\mathbf{Y}|\boldsymbol{\mathbf{X},\Theta_0})$ 
    \item Compute posterior expectation of log likelihood:\\ $LL(\boldsymbol{\Theta})=\mathbb{E}_{\mathbf{Y}|\mathbf{X},\boldsymbol{\Theta_0}}\log P(\mathbf{X},\mathbf{Y}|\boldsymbol{\Theta})$
    \item Update parameters:\\
    $\boldsymbol{\Theta_0}\leftarrow \argmax_{\boldsymbol{\Theta}}LL(\boldsymbol{\Theta})$\\
    Return $\boldsymbol{\Theta_0}$ if it's converged, otherwise goto (b)
\end{enumerate}

Among the four steps above, step (b) that computes the joint posterior distribution of all sample classes is practicallly infeasible due to the large number of hidden class nodes that are non-i.i.d. Fortunately, it is not necessary to compute the entire joint posterior distribution of all sample classes  $P(\mathbf{Y}|\boldsymbol{\mathbf{X},\Theta_0})$. In fact, we only need the marginal posterior distribution of a node's and its parents' classes for non-leaf nodes, as well as the marginal posterior distribution of a node's class for leaf nodes. The reason can be explained through the expression of the posterior expectation of log likelihood in Equation~\ref{eq:postexpll}. 
\begin{equation}\label{eq:postexpll}\small
\begin{split}
LL(\boldsymbol{\Theta}) & =\mathbb{E}_{\mathbf{Y}|\mathbf{X},\boldsymbol{\Theta_0}}\log P(\mathbf{X},\mathbf{Y}|\boldsymbol{\Theta})\\
&=\mathbb{E}_{\mathbf{Y}|\mathbf{X},\boldsymbol{\Theta_0}}\log\left\{ \prod_{n=1}^N P(\mathbf{x}_n|y_n,\boldsymbol{\Theta}) \prod_{n=1}^NP(y_n|y_{k\in\mathcal{P}_n},\boldsymbol{\Theta})\right\}\\
&=\sum\limits_{\mathbf{Y}}{ P(\mathbf{Y}|\mathbf{X},\boldsymbol{\Theta_0}) 
\sum_{n=1}^{N}{\left\{\log{P(\mathbf{x}_n|y_n,\boldsymbol{\Theta})}+\log{P(y_n|y_{k\in\mathcal{P}_n},\boldsymbol{\Theta})}\right\}}}\\
& =\sum_{n=1}^{N}\log{P(\mathbf{x}_n|y_n,\boldsymbol{\Theta})}\\
&\quad\quad+\sum_{n=1}^{N}~\sum_{y_n,y_{k\in\mathcal{P}_n}}\log{P(y_n|y_{k\in\mathcal{P}_n},\boldsymbol{\Theta})}P(y_n,y_{k\in\mathcal{P}_n}|\mathbf{X},\boldsymbol{\Theta_0}) \\
\end{split}
\end{equation}
Note that for leaf node, $\mathcal{P}_n=\emptyset$, and the last term in the last line of above equation is degraded, $\log{P(y_n|y_{k\in\mathcal{P}_n},\boldsymbol{\Theta})}P(y_n,y_{k\in\mathcal{P}_n}|\mathbf{X},\boldsymbol{\Theta_0})=\log{P(y_n|\boldsymbol{\Theta})}P(y_n|\mathbf{X},\boldsymbol{\Theta_0})$.

To compute the marginal posterior distribution $P(y_n,y_{k\in\mathcal{P}_n})$ and $P(y_n)$, we propose to use the message propagation method based on the sum and product algorithm~\cite{kschischang2001factor,ronen1995parameter}. 

Figure~\ref{fig:message} illustrates the recursive message propagation process on our HMT model. 
Specifically, forward message propagation from leaves to root is based on Equation~\ref{eq:hmtforwardin} and Equation~\ref{eq:hmtforwardout}, where $f_n^i(y_n)$ and $f_n^o(y_n)$ are the incoming message into and outgoing message from a hidden class node $y_n$ respectively. 
\begin{equation}\label{eq:hmtforwardin}
    f_n^i(y_n) =
            \begin{cases}
                \quad\quad\quad\quad\quad\quad P(y_n) & \text{if } y_n \text{ is leaf}\\
                \sum\limits_{y_{k\in \mathcal{P}_n}}P(y_n|y_{k\in \mathcal{P}_n})\prod\limits_{k\in \mathcal{P}_n}f_k^o(y_k) & \text{otherwise}
            \end{cases}
\end{equation}
\begin{equation}\label{eq:hmtforwardout}
    f_n^o(y_n) = f_n^i(y_n)P(\mathbf{x}_n|y_n)
\end{equation}
Backward message propagation from root to leaves also follows a recursive process, as shown in Equation~\ref{eq:hmtbackwardin} and Equation~\ref{eq:hmtbackwardout}, where $g_n^i(y_n)$ and $g_n^o(y_n)$ are the incoming and outgoing messages for class node $y_n$ respectively. The main difference from forward propagation is that when computing incoming message $g_n^i(y_n)$, we need to multiply not only outgoing message from a child node and class transitional probability, but also outgoing messages from sibling nodes in the forward propagation (also illustrated in Figure~\ref{fig:message}(b)).
\begin{equation}\label{eq:hmtbackwardin}
g_n^i(y_n) =
    \begin{cases}
        \quad\quad\quad\quad\quad\quad1 & \text{if } y_n \text{ is root}\\
        \sum\limits_{y_{c_n},y_{k\in \mathcal{S}_n}}g_{c_n}^oP(y_{c_n}|y_n,y_{k\in \mathcal{S}_n})\prod\limits_{k\in \mathcal{S}_n}f_k^o(y_k) & \text{otherwise}
    \end{cases}    
\end{equation}
\begin{equation}\label{eq:hmtbackwardout}
    g_n^o(y_n) = g_n^i(y_n)P(\mathbf{x}_n|y_n)
\end{equation}


After both forward and backward message propagation, we can compute marginal posterior distribution of hidden class variables based on the following theorem. 
\newtheorem{theorem}{Theorem}
\begin{theorem}
The unnormalized marginal posterior distribution of the class of a leaf node, as well as the classes of a non-leaf node with parents can be computed by (9) and (10) respectively. Their normalized marginal posterior distributions can be computed by (11) and (12) respectively.
\begin{equation}\label{eq:unnormalizedMarginY}
P^\prime(y_n|\mathbf{X},\boldsymbol{\Theta_0}) = f_n^i(y_n) g_n^i(y_n) P(\mathbf{x}_n|y_n)    
\end{equation}        
\begin{equation}\label{eq:unnormalizedMarginYYp}
P^\prime(y_n, y_{k\in \mathcal{P}_n}|\mathbf{X},\boldsymbol{\Theta_0})= \prod\limits_{k\in \mathcal{P}_n}f_k^o(y_k) g_n^o(y_n)  P(y_n | y_{k\in \mathcal{P}_n})  
\end{equation}
\begin{equation}\label{eq:normalizedMarginY}
P(y_n|\mathbf{X},\boldsymbol{\Theta_0})\leftarrow \frac{P^\prime(y_n|\mathbf{X},\boldsymbol{\Theta_0})}{\sum\limits_{y_n}P^\prime(y_n|\mathbf{X},\boldsymbol{\Theta_0})}    
\end{equation}
\begin{equation}\label{eq:normalizedMarginYYp}
P(y_n, y_{k\in \mathcal{P}_n}|\mathbf{X},\boldsymbol{\Theta_0})\leftarrow\frac{P^\prime(y_n, y_{k\in \mathcal{P}_n}|\mathbf{X},\boldsymbol{\Theta_0})}{\sum\limits_{y_n,y_{k\in \mathcal{P}_n}}P^\prime(y_n, y_{k\in \mathcal{P}_n}|\mathbf{X},\boldsymbol{\Theta_0})}    
\end{equation}
\end{theorem}    
\begin{proof}
Detailed proof is in the {\bf Appendix} at the end of this paper.
\end{proof}

After computation of marginal posterior distribution, we can update model parameters by maximizing the posterior expectation of log likelihood (the maximization or M step in EM). Taking the marginal posterior distributions in (11) and (12) above as well as parameters for probabilities in (2) and (3) into the posterior expectation of log likelihood in (4), we can easily get the following parameter update formulas.
\begin{equation}\label{eq:updaterho}
   \rho = \frac{\sum\limits_{n|\mathcal{P}_n\neq\emptyset}{\sum\limits_{y_n}\sum\limits_{y_{\mathcal{P}_n}}{y_{\mathcal{P}_n} y_nP(y_n, y_{\mathcal{P}_n}|\mathbf{X}, \boldsymbol{\Theta_0})}}} 
 {\sum\limits_{n|\mathcal{P}_n\neq\emptyset}{\sum\limits_{y_n}\sum\limits_{y_{\mathcal{P}_n}}{y_{\mathcal{P}_n}P(y_n, y_{\mathcal{P}_n}|\mathbf{X}, \boldsymbol{\Theta_0})}}} 
\end{equation}

\begin{equation}\label{eq:updatepi}
  \pi = \frac{\sum\limits_{n|\mathcal{P}_n=\emptyset}{\sum\limits_{y_n}{ y_n P(y_n|\mathbf{X}, \boldsymbol{\Theta_0})}}} {\sum\limits_{n|\mathcal{P}_n=\emptyset}{\sum\limits_{y_n}{P(y_n|\mathbf{X}, \boldsymbol{\Theta_0})}}}  
\end{equation}

\begin{equation}\label{eq:updatemu}
    \mu_c = \frac{\sum\limits_{n} \mathbf{x}_n P(y_n = c|\mathbf{X},\boldsymbol{\Theta_0})} {\sum\limits_{n}  P(y_n = c|\mathbf{X},\boldsymbol{\Theta_0})},c = {0, 1}
\end{equation}

\begin{equation}\label{eq:updatesigma}
    \Sigma_c = \frac{\sum\limits_{n} (\mathbf{x}_n - \boldsymbol{\mu}_c) (\mathbf{x}_n - \boldsymbol{\mu}_c)^T P(y_n = c|\mathbf{X},\boldsymbol{\Theta_0})} {\sum\limits_{n} P(y_n = c|\mathbf{X},\boldsymbol{\Theta_0})}, c = {0, 1}
\end{equation}

Algorithm~\ref{alg:learning} summarizes the entire parameter learning process. First, we initialize the set of parameters either with random values within reasonable range or with initial estimates based on training samples (e.g., the mean and covariance of features in each class). After parameters are initialized, the algorithm starts the iteration till parameters converge. In each iteration, it propagates messages first from leaves to root (steps 4-5) and then from root to leaves (steps 6-7). Marginal posterior distribution of node classes are then computed (steps 8-9). Based on this, the algorithm updates parameters (step 10). 

\begin{algorithm}
\caption{EM Algorithm for Hidden Markov Tree}
\label{alg:learning}
\begin{algorithmic}[1]
\REQUIRE\quad\\
$\bullet$ $\mathbf{X}=[\mathbf{x}_1,...,\mathbf{x}_N]^T$: {cell sample feature matrix}\\
$\bullet$ $\mathcal{T}$: {a reverse tree for spatial dependency}\\
$\bullet$ $\epsilon$: parameter convergence threshold\\
\ENSURE\quad\\
$\bullet$ $\boldsymbol{\Theta}=\{\rho, \pi, \boldsymbol{\mu}_c, \boldsymbol{\Sigma}_c|c=0,1 \}$: set of model parameters\\
\STATE Initialize $\boldsymbol{\Theta_0}$, $\boldsymbol{\Theta}$
\WHILE{$\|\boldsymbol{\Theta_0}-\boldsymbol{\Theta}\|_\infty>\epsilon$}
    \STATE $\boldsymbol{\Theta_0}\leftarrow\boldsymbol{\Theta}$
    \FORALL{$y_n$ from leaf to root} 
    \STATE Compute messages $f_n^i(y_n),f_n^o(y_n)$ by (5)-(6)
    \ENDFOR
    \FORALL{$y_n$ from root to leaf} 
    \STATE Compute messages $g_n^i(y_n),g_n^o(y_n)$ by (7)-(8)
    \ENDFOR
    \FORALL{$y_n, 1\leq n \leq N$} 
    \STATE // Compute marginal distributions:\\ 
    $\quad P(y_n|\mathbf{X},\boldsymbol{\Theta_0}), P(y_n,y_{k\in \mathcal{P}_n}|\mathbf{X},\boldsymbol{\Theta_0})$ by (9)-(12)
    \ENDFOR
    \STATE Update $\boldsymbol{\Theta}$ based on marginal distributions:\\         
    $\quad\boldsymbol{\Theta}\leftarrow \argmax\limits_{\boldsymbol{\Theta}}\mathbb{E}_{\mathbf{Y}|\mathbf{X},\boldsymbol{\Theta_0}}\log P(\mathbf{X},\mathbf{Y}|\boldsymbol{\Theta})$ by (13)-(16)
\ENDWHILE
\RETURN $\Theta$
\end{algorithmic}
\end{algorithm}

\emph{Time complexity}: The cost of Algorithm~\ref{alg:learning} mainly comes from the iterations. In each iteration, message propagation is done through tree traversal, which costs $O(N)$ ($N$ is the total number of samples or tree nodes). It can also be seen easily that marginal probability computation and parameter update both have costs of $O(N)$. Thus, the total cost is $O(N\cdot I)$, where $I$ is the number of iterations.

\emph{\bf Is the model unsupervised or semi-supervised?} From discussions above, it is possible to learn HMT parameters in an unsupervised manner without training class labels. However, this relies on strong assumptions on data distributions. Particularly, it requires samples in different classes to be somehow distinguishable merely based on their feature distribution ($P(\mathbf{x}_n|y_n)$), since class transitional probability in dependency tree only enforces a partial order constraint between class nodes. This assumption can be violated in many real world applications where different classes cannot be easily distinguished via unsupervised feature clustering. In such cases, we can utilize training samples with class labels to initialize parameters of $P(\mathbf{x}_n|y_n)$, i.e., $\{\boldsymbol{\mu}_c, \boldsymbol{\Sigma}_c|c=0,1\}$, by maximum likelihood estimation. In this way, initialized probability $P(\mathbf{x}_n|y_n)$ can better estimate class marginal distribution before iterations. This makes the model learning semi-supervised~\cite{zhu2005semi}.

\subsection{Class Inference}
After learning all model parameters, we can infer hidden class variables by maximizing the overall probability. In the traditional hidden Markov model, inference on hidden variables are done through the Viterbi algorithm~\cite{rabiner1989tutorial} based on dynamic programming. However, its computational cost is still very high for a large number of nodes (e.g., up to hundreds of millions). To address this challenge, we propose a greedy algorithm  that guarantees correctness based on the partial order class constraint. Taking the logarithm of joint probability in Equation~\ref{eq:joint}, we have the objective function below.
\begin{equation}\label{eq:logp}
\log P(\mathbf{X},\mathbf{Y})= \sum_{n=1}^N \log P(\mathbf{x}_n|y_n) + \sum_{n=1}^N\log P(y_n|y_{k\in\mathcal{P}_n})
\end{equation}

The goal of class inference is to assign a class label to each tree node such that the overall sum of log probability terms is maximized. Each term in the summation can be considered as a reward. For instance, $\log P(\mathbf{x}_n|y_n)$ is the reward for assigning class $y_n$ to node $n$ (i.e., node reward), $\log P(y_n|y_{k\in\mathcal{P}_n})$ is the reward for assigning class $y_n$ and $y_{k\in\mathcal{P}_n}$ to node $n$ and its parents respectively (i.e., edge reward). Thus, class inference in HMT becomes a \emph{node coloring problem}. Our goal is to find a node coloring to maximize the overall sum of rewards. In addition, the color must follow a partial order constraint, e.g., \emph{dry} nodes cannot follow \emph{flood} nodes, because otherwise, $P(y_n|y_{k\in\mathcal{P}_n})=0$. 
\begin{figure}[h]
\centering
\includegraphics[width=2.2in]{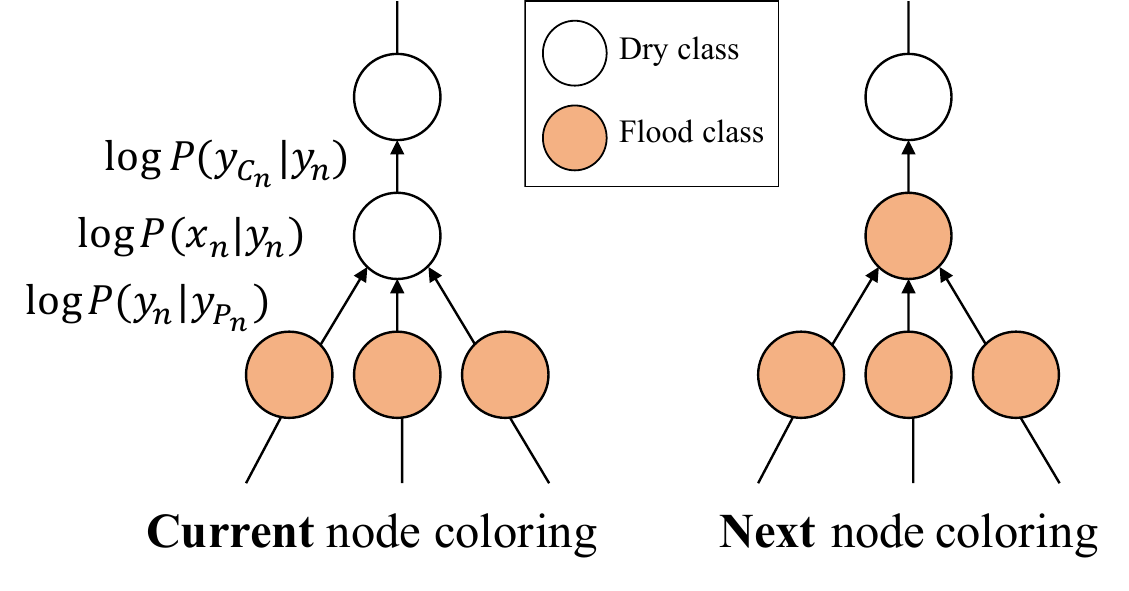}
\caption{Illustration of class inference process}
\label{fig:infer}
\end{figure}
Therefore, we can enumerate all feasible node coloring through one bottom-up tree traversal, as described in Algorithm~\ref{alg:inference}. We can initialize all node color as class $0$ (negative class, e.g., \emph{dry}), and gradually changed node colors from class $0$ to class $1$  from leaves to the root. When we change the color of a node, only the reward of the node itself, as well as the rewards of edges between the nodes to its parents and child will be updated, as illustrated in Figure~\ref{fig:infer}. Thus, we can easily compute the gain of rewards when updating node colors ($\Delta_{LL}$), and maintain the current cumulative gain ($g_{cur}(n)$) as well as the maximum cumulative gain  ($g_{max}$) that we've come across till the node so far. When we reach the root node, the maximum overall gain of rewards has been recorded. We can traverse the tree again to find its corresponding node coloring.

\begin{algorithm}
\caption{Class Inference for Hidden Markov Tree}
\label{alg:inference}
\begin{algorithmic}[1]
\REQUIRE\quad\\
$\bullet$ $\mathcal{T}$: {reverse tree for spatial dependency}\\
$\bullet$ $\boldsymbol{\Theta}=\{\rho, \pi, \theta, \boldsymbol{\mu}_c, \boldsymbol{\Sigma}_c|c=0,1 \}$: set of model parameters\\
\ENSURE\quad\\
$\bullet$ $\mathbf{Y}=[y_1,...,y_n]$: {inferred classes for all hidden nodes}
\STATE Initialize $y_n\leftarrow 0$ for $1\leq n \leq N$
\STATE Initialize $g_{cur}(n)\leftarrow 0$ for $1\leq n \leq N$
\STATE Initialize $g_{max}(n)\leftarrow 0$ for $1\leq n \leq N$
\FORALL {node $n$ in topological order from leaf to root}
    \STATE $y_n\leftarrow 1$
    \STATE 
    {\small $\Delta_{LL}\leftarrow\log\left( P(\mathbf{x_n}|y_n)P(y_{c_n}|y_n,y_{k\in\mathcal{S}_n})P(y_n|y_{k\in\mathcal{P}_n})\right) \bigg|^{y_n=1}_{y_n=0}$}\\
    // $y_{k\in\mathcal{P}_n}=\emptyset$ for leaf node $n$
    \STATE $g_{cur}(n)\leftarrow \sum\limits_{k\in\mathcal{P}_n}g_{cur}(k)+\Delta_{LL}$
    \STATE $g_{max}(n)\leftarrow \sum\limits_{k\in\mathcal{P}_n}g_{max}(k)$
    \IF{$g_{max}(n)<g_{cur}(n)$}
    \STATE $g_{max}(n)\leftarrow g_{cur}(n)$
    \ENDIF
\ENDFOR
\STATE Do breadth first tree traversal to find the frontier of maximum $g_{max}$
\STATE Set $y_n\leftarrow 0$ for nodes above the frontier
\RETURN $\mathbf{Y}=[y_1,...,y_n]$, the class labels of all nodes
\end{algorithmic}
\end{algorithm}

\emph{Time complexity analysis}: The initialization steps cost $O(N)$, where $N$ is the number of samples or tree nodes. Each iteration of the for loop has a constant cost, making the total cost $O(N)$. Similarly, the breadth first traversal and re-coloring in last step cost $O(N)$. Thus, the entire algorithm has a cost of $O(N)$.

\section{Experimental Evaluation}\label{sec:eval}
In this section, we compared our proposed method with baseline methods on both synthetic dataset and two real world datasets in classification performance. We also evaluated the computational scalability of our method on synthetic data with different sizes. Experiments were conducted on a Dell workstation with Intel(R) Xeon(R) CPU E5-2687w v4 @ 3.00GHz, 64GB main memory. Candidate classification methods for comparison include:
\begin{itemize}
    \item {\bf Non-spatial classifiers with raw features}: We tested decision tree ({\bf DT}), random forest ({\bf RF}), maximum likelihood classifier ({\bf MLC}), and gradient boosted tree ({\bf GBM}) in R packages on {\bf raw} features (including red, green, blue spectral bands respectively). 
    \item {\bf Non-spatial classifiers with preprocessed features}: We tested {\bf DT}, {\bf RF} and {\bf MLC} with additional elevation feature ({\bf elev.}) We do not include {\bf GBM} due to space limit.
    \item {\bf Non-spatial classifier with post-processing label propagation (LP):} We also tested {\bf DT}, {\bf RF} and {\bf MLC} on raw features but with post-processing on predicted classes via label propagation~\cite{zhu2002learning}. We used 4-neighborhood. We do not include {\bf GBM} due to space limit.
    \item {\bf Transductive SVM:} Since our method utilizes features of test samples, we included Transductive SVM (SVM-Light~\cite{Joachims99a}), a semi-supervised tranductive method for fair comparison.      
    \item {\bf Markov random field (MRF):} We used open source implementation~\cite{mrfsource} based on the graph cut method~\cite{szeliski2006comparative}.
    \item {\bf Hidden Markov Tree (HMT):} We implemented our HMT method in C++.
\end{itemize}
Unless specified otherwise, we used default parameters in open source tools for baseline methods.

\subsection{Synthetic Data}
We first evaluated our proposed approach on synthetic data. Specifically, we generated a regular grid with $1000$ by $1000$ pixels. Elevations and classes (flood, dry) of pixels are shown in Figure~\ref{fig:synthetic}(a-b). Feature values of pixels in two classes follow two one-dimensional Gaussian distributions with means $\mu_1=110,\mu_2=150$ and standard deviations $\sigma_1=\sigma_2=20$. To reflect the spatial autocorrelation effect, we generated one common feature value for a group of contiguous pixels in a coarse resolution ($50\times50$), as shown in Figure~\ref{fig:synthetic}(c). Training samples from two classes were generated based on the two Gaussian distributions of feature values. 
\begin{figure}[h]
\centering
\subfloat[Elevation map]{%
      \includegraphics[width=1in]{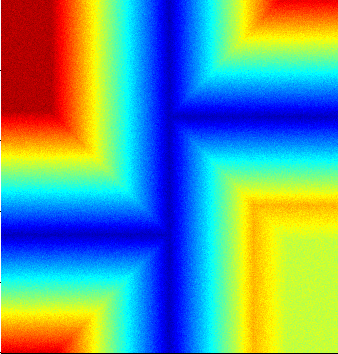}
}
\subfloat[Class map (red for flood, blue for dry)]{%
      \includegraphics[width=1in]{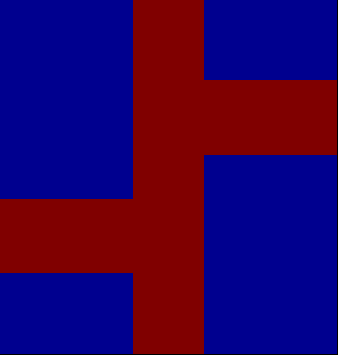}
}
\subfloat[Feature map]{%
      \includegraphics[width=1in]{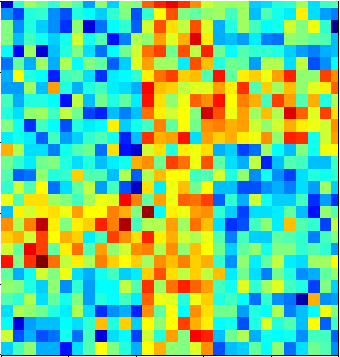}
}
\caption{Illustration of synthetic data (best viewed in color)}
\label{fig:synthetic}
\end{figure}

{\bf Computational scalability:} We measured the computational time costs of different components in our HMT algorithms on varying sizes of study area (from around $2$ million pixels to around $20$ million pixels). We also fixed the number of iterations as $3$ when running algorithms on different data sizes. The time costs were measured in the average of $10$ runs. Figure~\ref{fig:comp} shows the time costs of tree construction (Algorithm~\ref{alg:treebaseline}), parameter learning (Algorithm~\ref{alg:learning}), and class inference (Algorithm~\ref{alg:inference}) respectively. We can see that as the number of pixels increases, time costs of all algorithms are increasing. The parameter learning part takes the vast majority of time costs. Its time costs increase linearly with data sizes, because the message propagation in each iteration is done through tree traversal operations, which has a linear time complexity. Overall, our algorithms cost less than 5 minutes on a synthetic data with 20 million samples.

\begin{figure}[h]
\centering
\includegraphics[width=2in, angle=270]{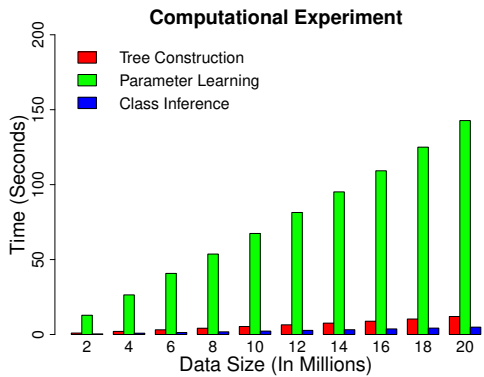}
\caption{Computational time costs of HMT on different data sizes}
\label{fig:comp}
\end{figure}

{\bf Classification performance:}
We compared the F-score of different methods on test pixels with different parameter settings of synthetic data generation. We exclude pre-processing and post-processing methods because our synthetic data generation cannot simulate the real feature textures. In the first setting, we conducted comparison on varying numbers of training pixels from $10$, $1000$, to $10000$. Results in Figure~\ref{fig:syntheticCrossMethods}(a) showed that the classification performance of different methods were relatively stable (easily reaching plateau) for different training set sizes. The reason was probably that one dimensional Gaussian distributions on feature values in two classes were very easy to learn. In the second setting, we fixed other parameters and varied the standard deviations $\sigma_1,\sigma_2$ of feature values in two classes. The higher the values were, the more confusion (Bayes error) there were between two classes. Results of different methods in Figure~\ref{fig:syntheticCrossMethods}(b) showed that as $\sigma_1,\sigma_2$ increase, the classification performance of all methods degraded, but our HMT model persistently outperformed other baseline methods, due to incorporating anisotropic spatial dependency across locations.

\begin{figure}[h]
\centering
\subfloat[]{%
      \includegraphics[width=2in,angle=270]{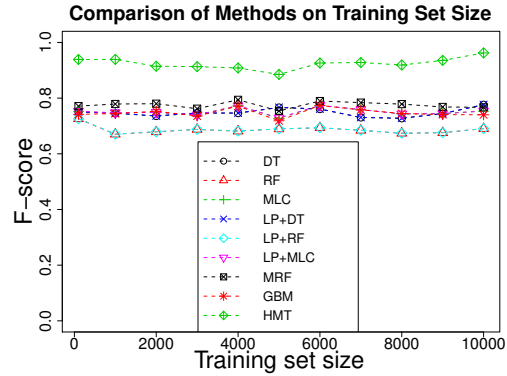}
}\\
\subfloat[]{%
      \includegraphics[width=2in,angle=270]{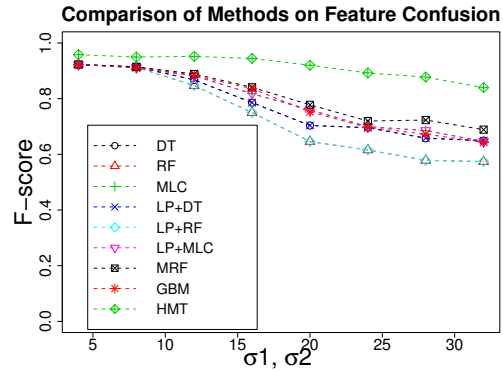}
}
\caption{Classification performance comparison across methods on synthetic data}
\label{fig:syntheticCrossMethods}
\end{figure}



\subsection{Hurricane Mathew Floods 2016}
Here we validated our approach in flood inundation extent mapping during Hurricane Mathew, NC, 2016. We used high-resolution aerial imagery from NOAA National Geodetic Survey~\cite{ngs} as explanatory features (three spectral band features including red, green, blue), and digital elevation map from the University of North Carolina Libraries~\cite{ncsudem}. All imagery data were re-sampled into a resolution of 2 meters. A test region with 1743 by 1349 pixels was used. A training set with 10000 pixels (5000 \emph{dry} and 5000 \emph{flood}) were manually labeled outside the test region, and 94608 test pixels  (47092 \emph{dry}, 47516 \emph{flood}) were labeled within the test region.

\begin{table}[h]
\centering
\caption{Comparison on Hurricane Mathew Flood data}
\begin{tabular}{cccccc}
\hline
Classifiers & Class & Precision &Recall & F & Avg. F\\ \hline
\multirow{2}{*}{DT+Raw}&Dry&{0.62}&{0.84}&{0.71}&\multirow{2}{*}{0.65}\\ 
 &Flood&{0.76}&{0.48}&{0.59}&\\ \hline
\multirow{2}{*}{RF+Raw}&Dry&{0.59}&{0.96}&{0.73}&\multirow{2}{*}{0.61}\\ 
 &Flood&{0.90}&{0.33}&{0.49}&\\ \hline
\multirow{2}{*}{GBM+Raw}&Dry&{0.69}&{0.76}&{0.72}&\multirow{2}{*}{0.71}\\ 
 &Flood&{0.74}&{0.67}&{0.70}&\\ \hline  
\multirow{2}{*}{MLC+Raw}&Dry&{0.64}&{0.93}&{0.76}&\multirow{2}{*}{0.69}\\ 
 &Flood&{0.88}&{0.48}&{0.62}&\\ \hline 
\multirow{2}{*}{DT+elev.}&Dry&{0.99}&{0.55}&{0.71}&\multirow{2}{*}{0.76}\\ 
 &Flood&{0.69}&{0.99}&{0.82}&\\ \hline 
\multirow{2}{*}{RF+elev.}&Dry&{0.99}&{0.66}&{0.79}&\multirow{2}{*}{0.82}\\ 
 &Flood&{0.74}&{0.99}&{0.85}&\\ \hline 
\multirow{2}{*}{MLC+elev.}&Dry&{0.84}&{0.90}&{0.87}&\multirow{2}{*}{0.87}\\ 
 &Flood&{0.89}&{0.84}&{0.86}&\\ \hline 
\multirow{2}{*}{DT+LP}&Dry&{0.61}&{0.92}&{0.74}&\multirow{2}{*}{0.65}\\ 
 &Flood&{0.85}&{0.43}&{0.57}&\\ \hline 
\multirow{2}{*}{RF+LP}&Dry&{0.57}&{0.99}&{0.72}&\multirow{2}{*}{0.57}\\ 
 &Flood&{0.99}&{0.26}&{0.42}&\\ \hline 
\multirow{2}{*}{MLC+LP}&Dry&{0.64}&{0.97}&{0.77}&\multirow{2}{*}{0.69}\\ 
 &Flood&{0.95}&{0.46}&{0.62}&\\ \hline  
\multirow{2}{*}{MRF}&Dry&{0.62}&{0.99}&{0.76}&\multirow{2}{*}{0.67}\\ 
 &Flood&{0.98}&{0.41}&{0.58}&\\ \hline  
\multirow{2}{*}{TSVM}&Dry&{0.62}&{0.86}&{0.72}&\multirow{2}{*}{0.66}\\ 
 &Flood&{0.78}&{0.49}&{0.60}&\\ \hline  
\multirow{2}{*}{HMT}&Dry&{0.93}&{0.99}&{0.96}&\multirow{2}{*}{0.96}\\ 
 &Flood&{0.99}&{0.93}&{0.96}&\\ \hline  
\end{tabular}
\label{tab:compFlood}
\end{table}

{\bf Classification performance comparison:} We compared different methods on their precision, recall, and F-score on both the \emph{flood} class and the \emph{dry} class. Results were summarized in Table~\ref{tab:compFlood}. We can see that decision tree, random forest, and maximum likelihood classifier all perform poorly on raw features, with overall F-score less than 0.7. Adding post-processing through label propagation slightly impaired performance. For example, after adding post-processing (LP) into decision tree, the recall of the \emph{dry} class got better but the recall of \emph{flood} class got worse, probably due to over-smoothing of correctly classified \emph{flood} pixels into the \emph{dry} class. Markov random field and Transductive SVM had comparable results with decision tree. In contrast, adding elevation features improved the overall classification performance dramatically for decision tree, random forest, and maximum likelihood classifier. The reason is because most \emph{flood} pixels have lower elevation than \emph{dry} pixels. However, the performance on dry class was still quite inferior (0.86 to 0.9 in F-score) compared with our hidden Markov tree (0.95 in F-score), probably because classification models learned based on absolute elevation values cannot perfectly apply to the test region.

Some visualization of classification results were shown in Figure~\ref{fig:real1CaseStudy}. The spectral features and elevation values were shown in Figure~\ref{fig:real1CaseStudy}(a-b). Results of decision tree were in Figure~\ref{fig:real1CaseStudy}(c), which only identified \emph{flood} pixels with open surface, and mistakenly classified the vast majority of \emph{flood} pixels below tree canopies (the spectral features of trees indicated the \emph{dry} class if not considering spatial dependency with nearby pixels). In contrast, our HMT model correctly identified most of the \emph{flood} pixels, even if the flood water was below tree canopies. The reason is that our HMT incorporates the anisotropic spatial dependency across pixel locations (if a location is \emph{flood}, its nearby lower locations should also be \emph{flood}).
\begin{figure}[h]
\centering
\subfloat[High-resolution satellite imagery in North Carolina]{%
      \includegraphics[width=1.5in]{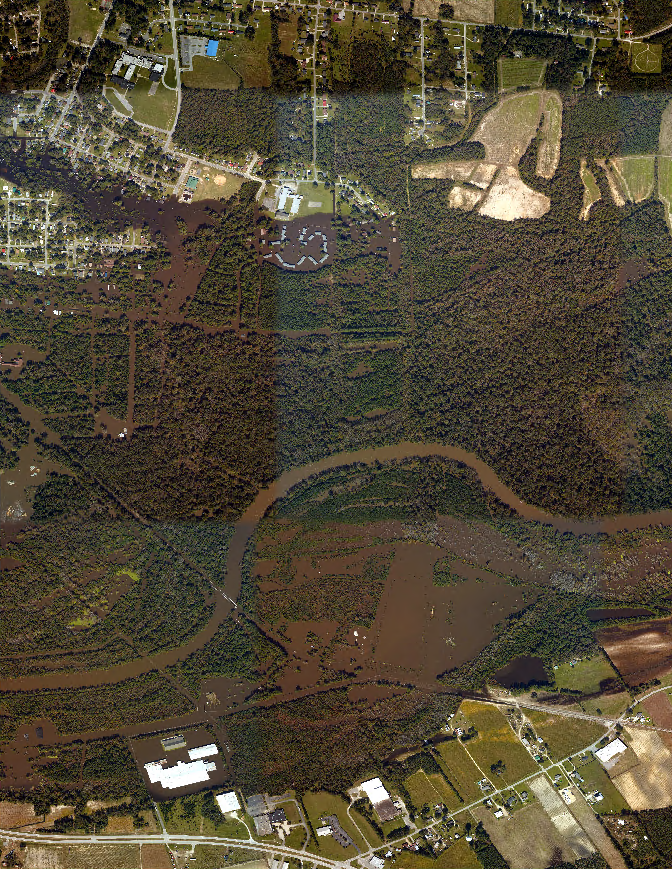}
}
\subfloat[Digital elevation]{%
      \includegraphics[width=1.5in]{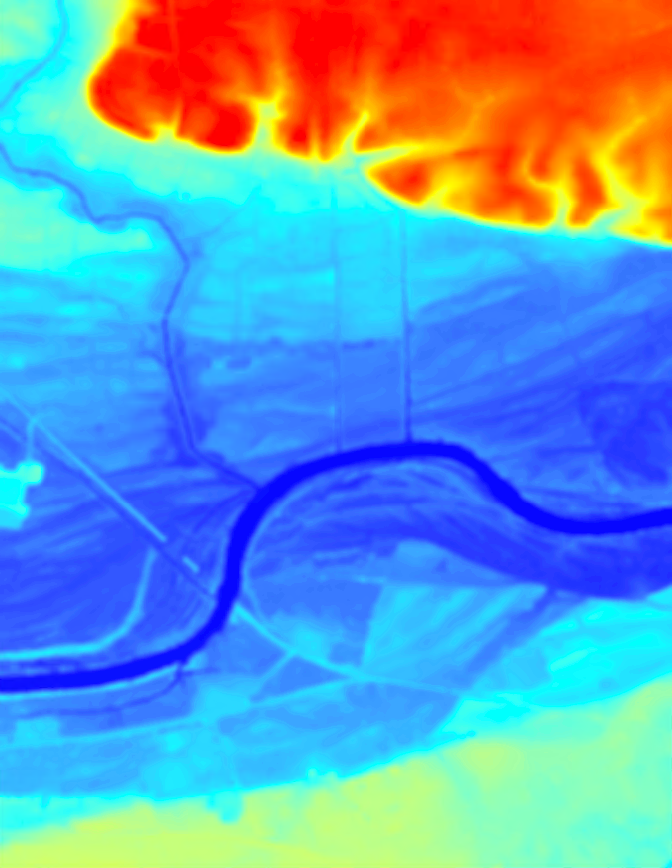}
}\\
\subfloat[Decision tree result]{%
      \includegraphics[width=1.5in]{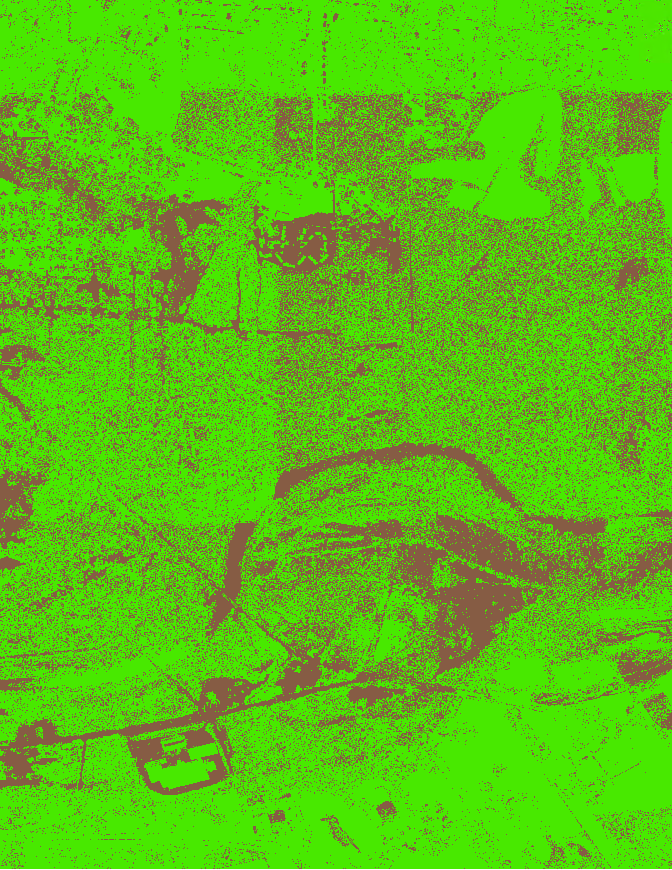}
}
\subfloat[HMT result]{%
      \includegraphics[width=1.5in]{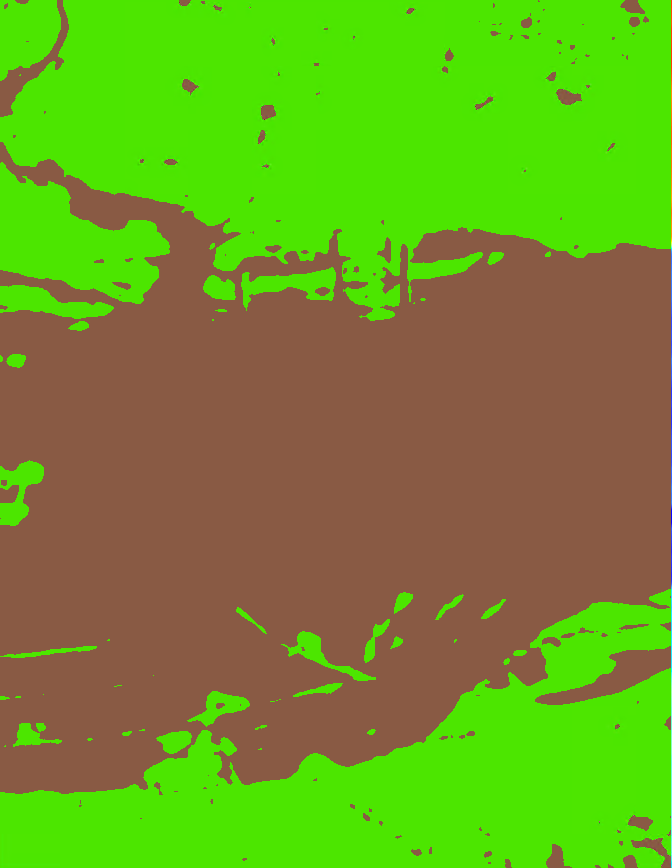}
}
\caption{Results on Mathew flood mapping (flood in brown, dry in green, best viewed in color)}
\label{fig:real1CaseStudy}
\end{figure}


{\bf Sensitivity of HMT to initial parameters:} We conducted sensitivity of our HMT model to different initial parameter values on prior class probability $\pi$ and class transitional probability $\rho$ (the parameters of $\{\boldsymbol{\mu}_c,\boldsymbol{\Sigma}_c|c=1,2\}$ were initialized based on maximum likelihood estimation on the training set). First, we fixed initial $\rho=0.99$ and varied initial $\pi$ from $0.1$ to $0.9$. Results of converged value of $\rho$ together with final F-score were shown in Figure~\ref{fig:realSensitivity}(a-b). It can be seen that our HMT model was quite stable with different initial $\pi$ values. Similarly, we fixed initial $\pi = 0.5$, and varied initial $\rho$ from $0.2$, $0.3$, to $0.99$. Results in 
Figure~\ref{fig:realSensitivity}(c-d) showed the same trend. In practice, we can select an initial $\pi$ value around $0.5$ and a relatively high initial $\rho$ value such as $0.9$ (because \emph{flood} pixels' neighbor pixel is very likely to be \emph{flood} due to spatial autocorrelation).

\begin{figure}[h]
\centering
\subfloat[]{%
      \includegraphics[width=1.3in, angle=270]{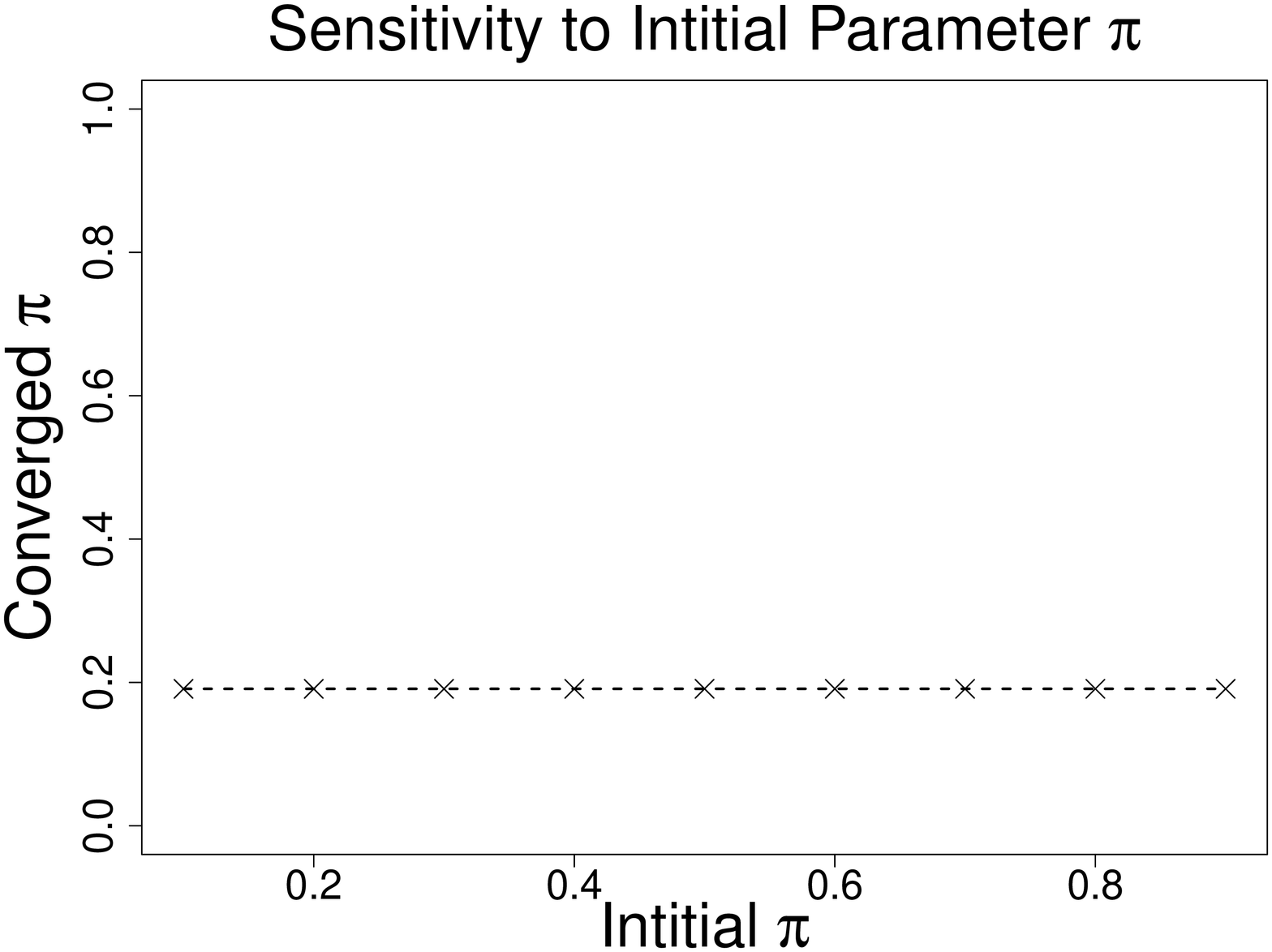}
}
\subfloat[]{%
      \includegraphics[width=1.3in, angle=270]{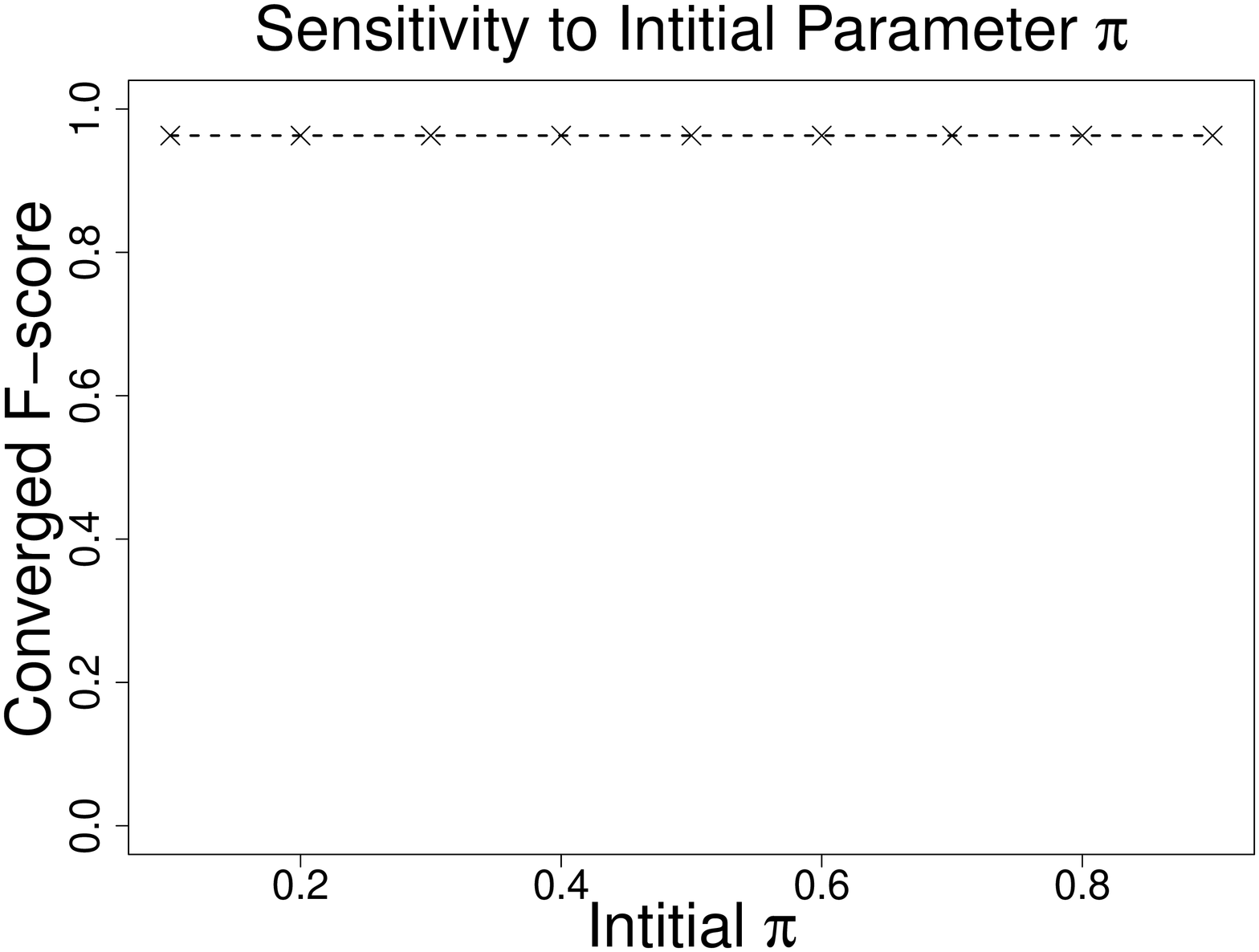}
}\\
\subfloat[]{%
      \includegraphics[width=1.3in, angle=270]{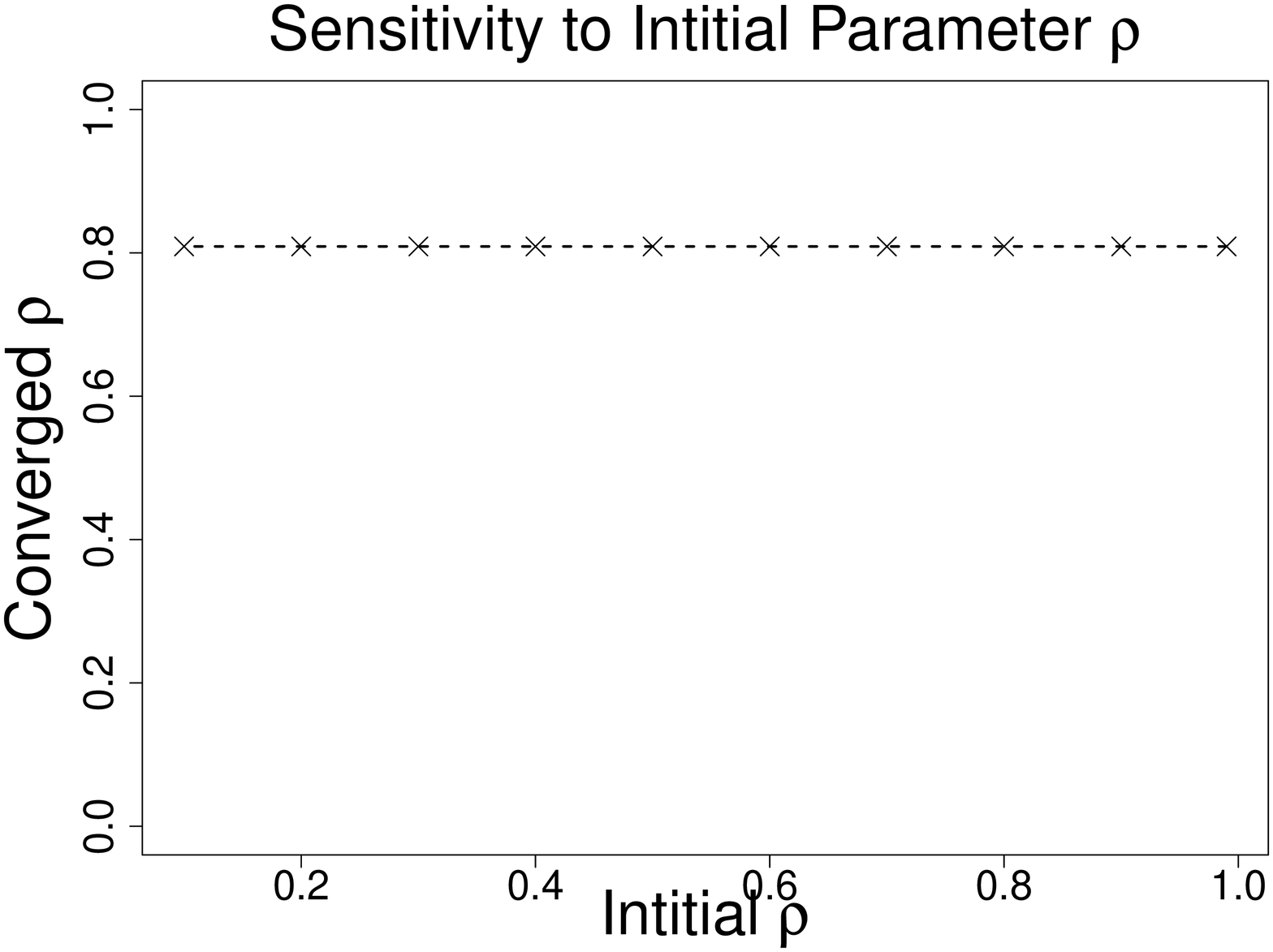}
}
\subfloat[]{%
      \includegraphics[width=1.3in, angle=270]{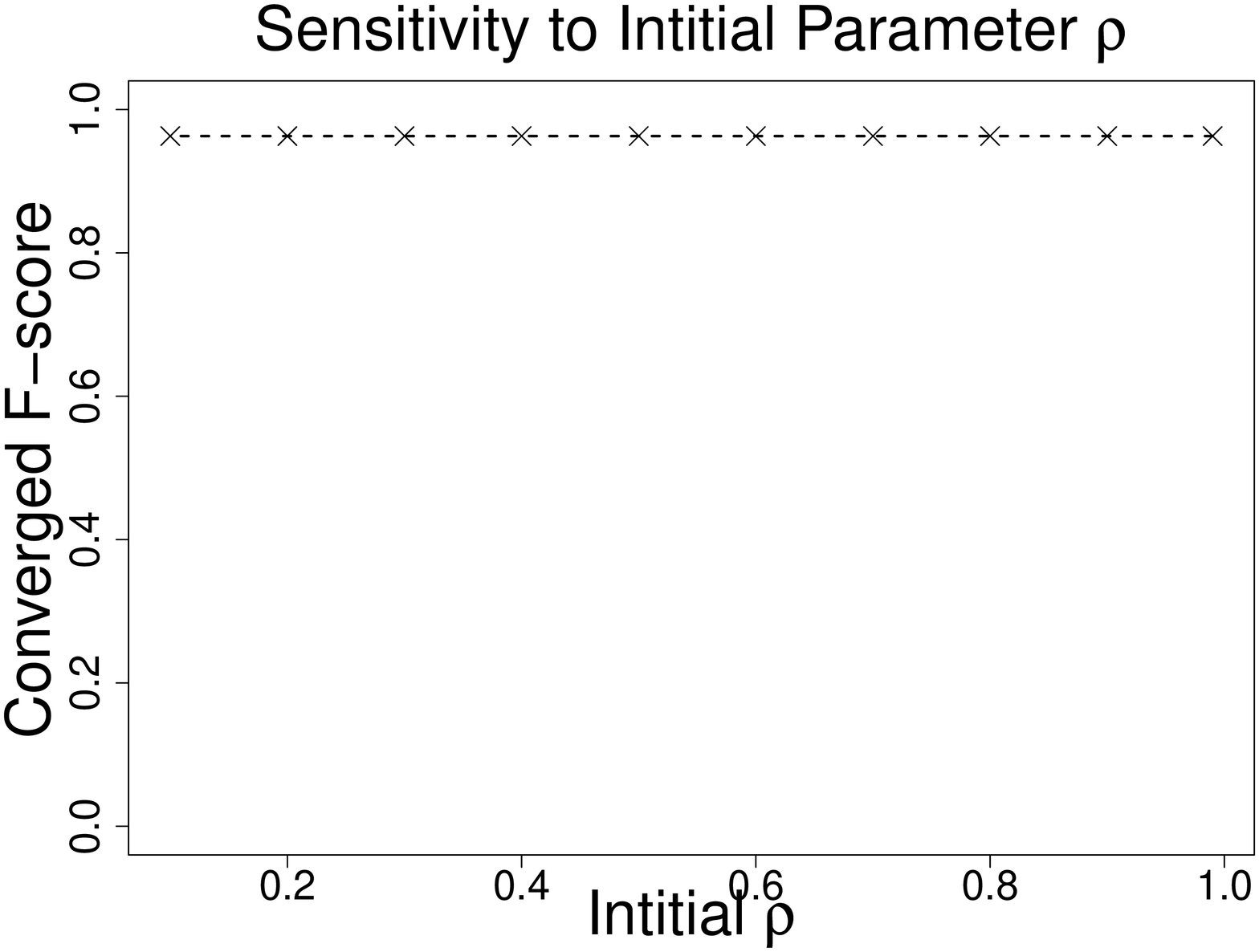}
}
\caption{Sensitivity of HMT to different initial parameters $\pi$ and $\rho$}
\label{fig:realSensitivity}
\end{figure}

{\bf Parameter iterations and convergence in HMT:} Here we fixed the initial $\pi=0.5$ and initial $\rho=0.99$, and measured the parameter iterations and convergence behavior. Our convergence threshold was set $0.001\%$. The parameter values of $\pi$ and $\rho$ at each iteration were summarized in Figure~\ref{fig:realIterations} (we omitted $\boldsymbol{\mu}_c,\boldsymbol{\Sigma}_c$ because there were too many variables). The parameters converged after 10 iterations. 

\begin{figure}[h]
\centering
\includegraphics[width=2.2in, angle=270]{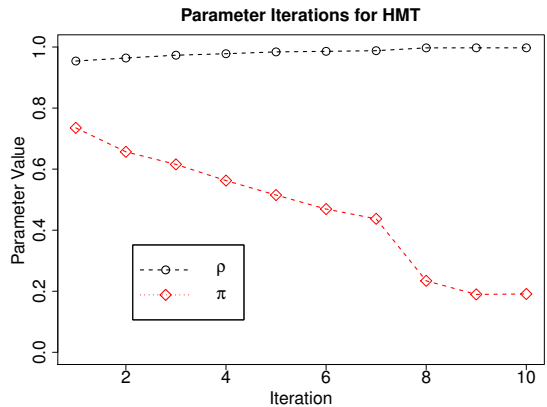}
\caption{Parameter iterations and convergence in HMT}
\label{fig:realIterations}
\end{figure}

\subsection{Hurricane Harvey Floods 2017 }
The high-resolution earth imagery we used were from Plant Labs. Inc. with red, green, and blue bands in 3 meter resolution, and the digital elevation data was from Texas natural resource management department. We manually collected a training set with 5000 flood samples and 5000 dry samples. We selected a test scene with 4174 rows and 4592 columns, within which we manually labeled 74305 flood samples, and 52658 dry samples as the test set.

We compared different methods on their precision, recall, and F-score on each class. Results were summarized in Table~\ref{tab:compReal2}. We can see that decision tree, random forest, and maximum likelihood classifier all perform poorly on raw features, with overall F-score less than 0.75. Adding additional elevation feature improved the classification accuracy slightly, but the overall F-score was still below 0.8. The reason may be that the absolute elevation values from training area are not as good an indicator for two classes as in the test area. Similarly, adding label propagation in post-processing and the MRF model barely improved classification performance compared with decision tree and random forest on raw features, because the errors were mostly systematic instead of salt-and-pepper noise. TSVM performed even worse than supervised methods without unlabeled samples, which was somehow surprising. In contrast to baseline methods, our hidden Markov tree achieved superior performance with around 0.95 F-score on both classes. Visualization of some results were shown in Figure~\ref{fig:real2CaseStudy}. We can see that our HMT model significantly outperformed decision tree on flood locations under tree canopies, similar to the results in Figure~\ref{fig:real1CaseStudy}.
\begin{table}
\centering
\caption{Comparison on Harvey Flood Data}
\begin{tabular}{cccccc}
\hline
Classifiers & Class & Precision &Recall & F & Avg. F\\ \hline
\multirow{2}{*}{DT+Raw}&Dry&{0.58}&{0.88}&{0.70}&\multirow{2}{*}{0.69}\\ 
 &Flood&{0.87}&{0.56}&{0.68}&\\ \hline
\multirow{2}{*}{RF+Raw}&Dry&{0.62}&{0.96}&{0.76}&\multirow{2}{*}{0.74}\\ 
 &Flood&{0.95}&{0.59}&{0.73}&\\ \hline
\multirow{2}{*}{GBM+Raw}&Dry&{0.56}&{0.80}&{0.66}&\multirow{2}{*}{0.66}\\ 
  &Flood&{0.80}&{0.56}&{0.66}&\\ \hline 
\multirow{2}{*}{MLC+Raw}&Dry&{0.63}&{0.93}&{0.75}&\multirow{2}{*}{0.74}\\ 
 &Flood&{0.93}&{0.61}&{0.73}&\\ \hline 
\multirow{2}{*}{DT+elev.}&Dry&{0.61}&{0.99}&{0.76}&\multirow{2}{*}{0.74}\\ 
 &Flood&{0.99}&{0.56}&{0.72}&\\ \hline 
\multirow{2}{*}{RF+elev.}&Dry&{0.66}&{0.99}&{0.79}&\multirow{2}{*}{0.79}\\ 
 &Flood&{0.99}&{0.65}&{0.78}&\\ \hline 
\multirow{2}{*}{MLC+elev.}&Dry&{0.65}&{0.97}&{0.78}&\multirow{2}{*}{0.77}\\ 
 &Flood&{0.97}&{0.62}&{0.76}&\\ \hline 
 \multirow{2}{*}{DT+LP}&Dry&{0.59}&{0.90}&{0.71}&\multirow{2}{*}{0.70}\\ 
  &Flood&{0.89}&{0.56}&{0.69}&\\ \hline 
 \multirow{2}{*}{RF+LP}&Dry&{0.63}&{0.97}&{0.76}&\multirow{2}{*}{0.75}\\ 
  &Flood&{0.97}&{0.59}&{0.74}&\\ \hline 
 \multirow{2}{*}{MLC+LP}&Dry&{0.63}&{0.94}&{0.76}&\multirow{2}{*}{0.75}\\ 
  &Flood&{0.93}&{0.61}&{0.74}&\\ \hline  
\multirow{2}{*}{MRF}&Dry&{0.63}&{0.94}&{0.75}&\multirow{2}{*}{0.74}\\ 
 &Flood&{0.94}&{0.61}&{0.74}&\\ \hline  
 \multirow{2}{*}{TSVM}&Dry&{0.55}&{0.67}&{0.60}&\multirow{2}{*}{0.63}\\ 
  &Flood&{0.72}&{0.61}&{0.66}&\\ \hline  
\multirow{2}{*}{HMT}&Dry&{0.91}&{0.98}&{0.94}&\multirow{2}{*}{0.95}\\ 
 &Flood&{0.98}&{0.93}&{0.95}&\\ \hline  
\end{tabular}
\label{tab:compReal2}
\end{table}

\begin{figure}[h]
\centering
\subfloat[High-resolution satellite imagery in Houston, TX]{%
      \includegraphics[width=1.5in]{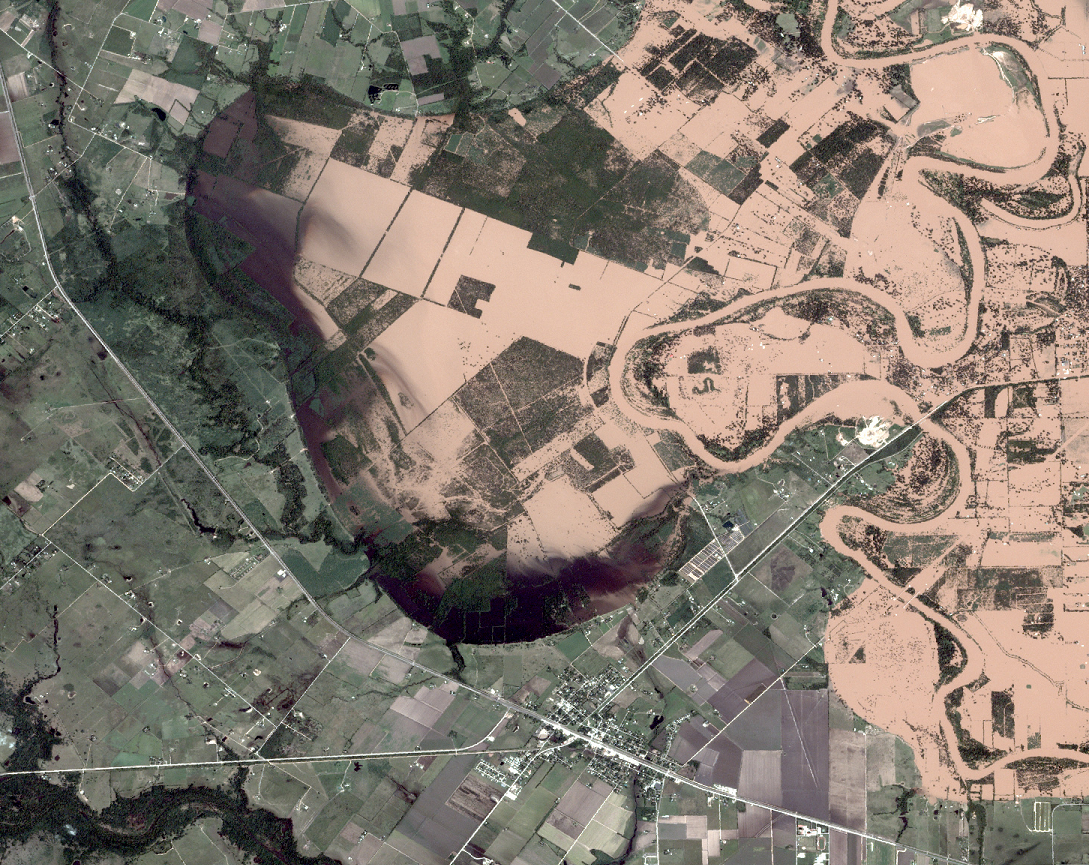}
}
\subfloat[Digital elevation]{%
      \includegraphics[width=1.5in]{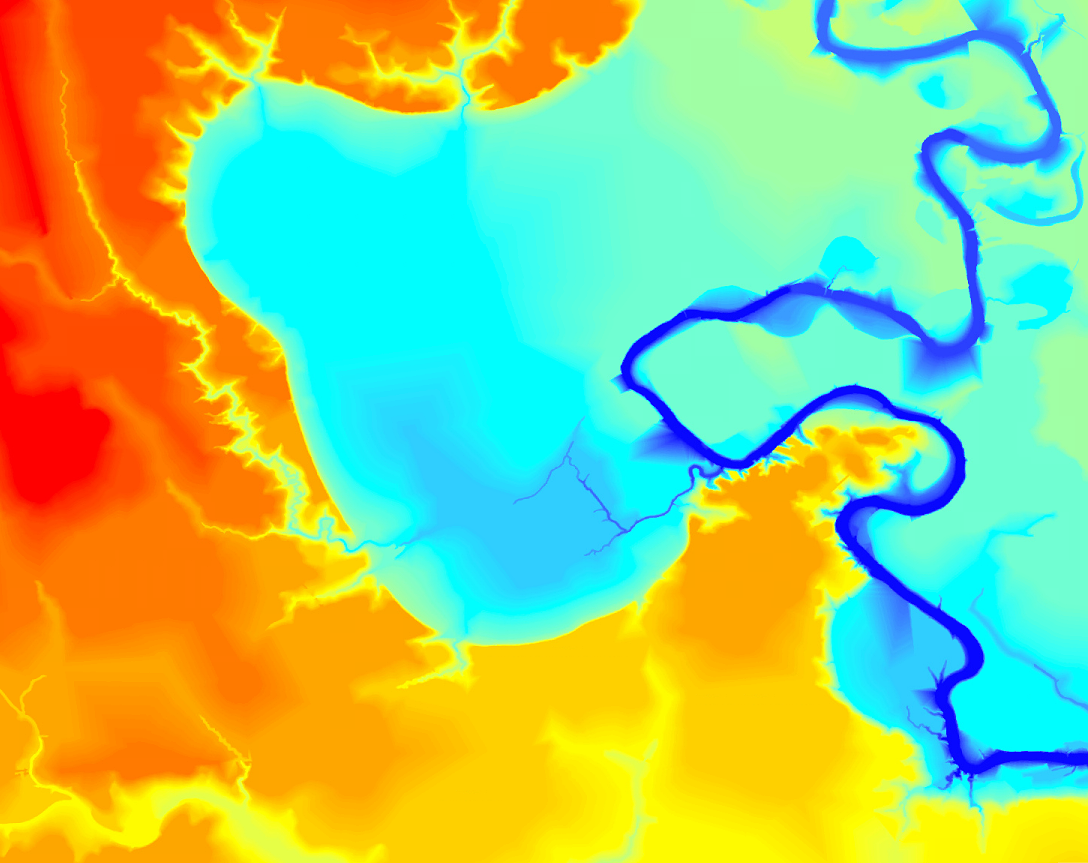}
}\\
\subfloat[Decision tree result]{%
      \includegraphics[width=1.5in]{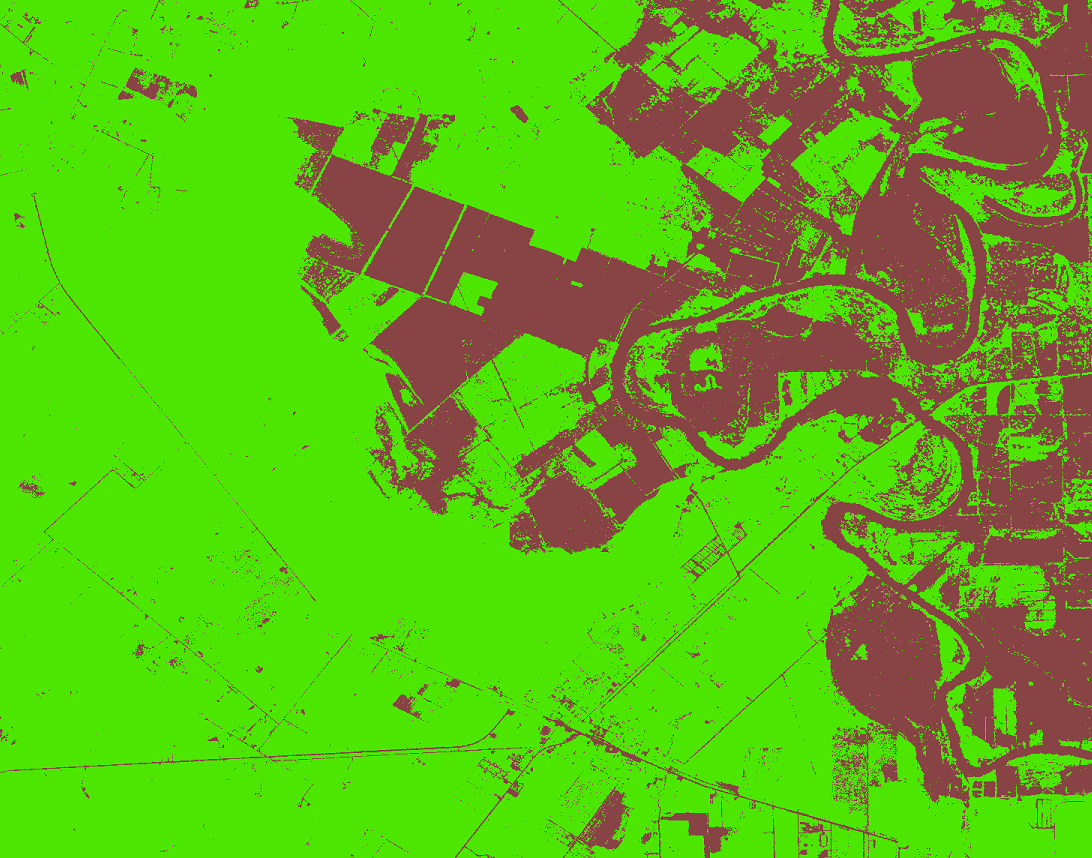}
}
\subfloat[HMT result]{%
      \includegraphics[width=1.5in]{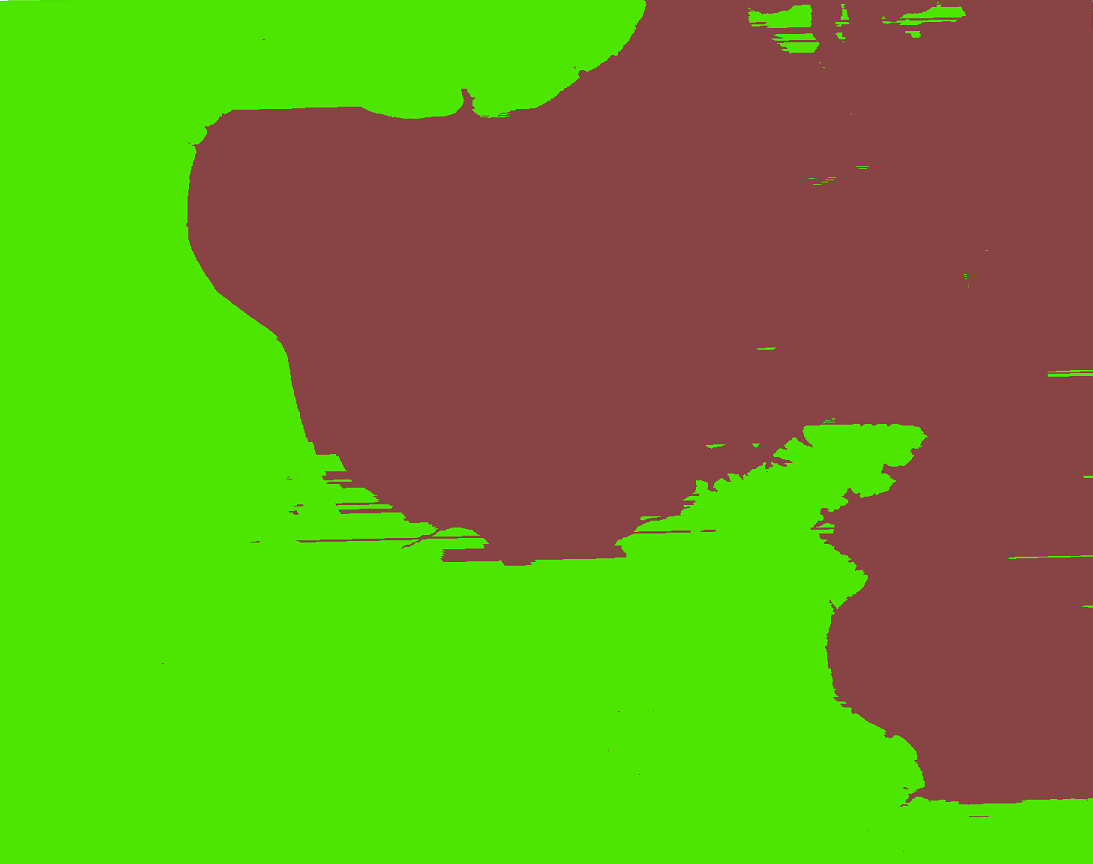}
}
\caption{Results on Harvey flood mapping (flood in brown, dry in green, best viewed in color)}
\label{fig:real2CaseStudy}
\end{figure}
\section{Related Work}

Over the years, various techniques have been developed to incorporate spatial properties into classification algorithms for earth imagery data. Many methods are based on preprocessing and post-processing, including neighborhood window filters~\cite{chan2005salt,esakkirajan2011removal}, spatial contextual variables and textures~\cite{puissant2005utility}, spatial autocorrelation statistics~\cite{jiang2015focal}, morphological profiling~\cite{benediktsson2005classification}, spatial-spectral classifiers~\cite{tarabalka2009spectral,wang2014semi} and object-based image analysis~\cite{hay2008geographic}. Markov random field model explicitly captures spatial dependency, but the dependency is undirected~\cite{li2009markov}. 
\cite{KhandelwalMK15} proposes a spatial classification model that captures directed spatial dependency on classes but assumes dependency to follow a total order. Deep learning methods have recently been applied to earth imagery classification~\cite{zhang2016deep} such as land cover mapping~\cite{jia2017incremental}, target recognition and scene identification. To the best of our knowledge, none of these existing works focus on incorporating anisotropic spatial dependency in partial order constraints, which is important in hydrological applications such as flood mapping.


Hidden Markov models have been extensively studied in the signal processing literature~\cite{rabiner1989tutorial}. Learning and inference of hidden Markov models are often based on EM algorithms and message propagation (sum-and-product algorithm)~\cite{kschischang2001factor}. \cite{ronen1995parameter} proposes a message propagation algorithm called ``upward-downward" on a dependence tree structure. \cite{crouse1998wavelet} proposes a wavelet-domain hidden Markov tree to model dependency in the two-dimensional time-frequency plane. The model is used to characterize properties of wavelet coefficients in signal processing such as clustering, persistence, and compression, which is dramatically different from our HMT model which captures spatial dependency in the geographic space based on topography.

\section{Conclusions and Future Works}
In this paper, we propose hidden Markov tree (HMT), an anisotropic spatial classification model for flood mapping on earth imagery. Compared with existing methods, our HMT model explicitly captures directed spatial dependency with partial order constraint. Partial order constraint is reflected in a reverse tree structure in the hidden class layer. We also proposed algorithms for reverse tree construction, parameter learning and class inference. Evaluations on both synthetic data and real world data shows that our HMT model is scalable to large data sizes, and can utilize the partial order spatial dependency to reduce classification errors due to class confusion. 

In future works, we plan to extend our HMT model for spatially non-stationary data. In this case, we need to generalize the tree structure in hidden class layer into poly-tree with each sub-tree for a spatial zone. 



\bibliographystyle{abbrv}
\bibliography{TKDE_HMT}

\appendix
\section{Proof of Theorems}
\setcounter{theorem}{0}
\begin{theorem}
The unnormalized marginal posterior distribution of the class of a leaf node, as well as the classes of a non-leaf node with parents can be computed by (9) and (10) respectively. Their normalized margin posterior distributions can be computed by (11) and (12) respectively.
\begin{equation}\label{eq:unnormalizedMarginY}
P^\prime(y_n|\mathbf{X},\boldsymbol{\Theta_0}) = f_n^i(y_n) g_n^i(y_n) P(\mathbf{x}_n|y_n)    
\end{equation}        
\begin{equation}\label{eq:unnormalizedMarginYYp}
P^\prime(y_n, y_{k\in \mathcal{P}_n}|\mathbf{X},\boldsymbol{\Theta_0})= \prod\limits_{k\in \mathcal{P}_n}f_k^o(y_k) g_n^o(y_n) P(y_n | y_{k\in \mathcal{P}_n})    
\end{equation}
\begin{equation}\label{eq:normalizedMarginY}
P(y_n|\mathbf{X},\boldsymbol{\Theta_0})\leftarrow \frac{P^\prime(y_n|\mathbf{X},\boldsymbol{\Theta_0})}{\sum\limits_{y_n}P^\prime(y_n|\mathbf{X},\boldsymbol{\Theta_0})}    
\end{equation}
\begin{equation}\label{eq:normalizedMarginYYp}
P(y_n, y_{k\in \mathcal{P}_n}|\mathbf{X},\boldsymbol{\Theta_0})\leftarrow\frac{P^\prime(y_n, y_{k\in \mathcal{P}_n}|\mathbf{X},\boldsymbol{\Theta_0})}{\sum\limits_{y_n,y_{k\in \mathcal{P}_n}}P^\prime(y_n, y_{k\in \mathcal{P}_n}|\mathbf{X},\boldsymbol{\Theta_0})}    
\end{equation}
\end{theorem}    

\begin{proof}
Detailed proof can be found below. Some symbols are defined in Table~\ref{tab:symbols}.
\begin{table}[h]
    \caption{List of symbols in the proof}
    \label{tab:symbols}
    \centering
    \begin{tabular}{|c|p{2.5in}|}\hline
    Symbol  &  Description\\
      $subtree(n)$   &All nodes in the subtree rooted at node $n$ except for the root node $n$  \\ \hline
      $tree(n)$   & All nodes in the subtree rooted at node $n$ \\ \hline
      $pre(n)$   &All nodes in the reverse tree excluding $tree(n)$ \\ \hline
      $passed(n)$   &All nodes in the reverse tree excluding $subtree(n)$ \\ \hline
    \end{tabular}
\end{table}

Forward message propagation, the statistical meanings of the message are as below
\begin{align}
    f_n^i(y_n) &= P(\mathbf{x}_{subtree(n)}, y_n)\label{eq:fi}\\
    f_n^o(y_n) &= P(\mathbf{x}_{tree(n)}, y_n)\label{eq:fo}
\end{align}
Here, $subtree(n)$ means all nodes in the subtree rooted at node $n$ except for the root node $n$ (i.e., all nodes visited before $n$ in its subtrees during bottom-up propagation), $tree(n)$ means all nodes in the subtree rooted at node $n$ (i.e., $tree(n)=\{n\}\cup subtree(n)$).

Base case: If $n$ is leaf node, 
\begin{align}
    f_n^i(y_n) &= P(y_n)\\
    \begin{split}
            f_n^o(y_n) &= f_n^i(y_n)P(\mathbf{x}_n|y_n)\\
                        &= P(\mathbf{x}_n, y_n)
    \end{split}
\end{align}

Thus, the statement Equation \ref{eq:fi} and \ref{eq:fo} are correct.

Induction step: Let $y_k$ be an internal node and and suppose Equation \ref{eq:fi} and \ref{eq:fo} is true for $n=k$. Then, based on topological order, let $y_n$ be the child of node $y_k$.
\begin{equation}\small
    \begin{split}
    f_n^i(y_n) &=\sum\limits_{y_{k\in \mathcal{P}_n}}P(y_n|y_{k\in \mathcal{P}_n}) \prod\limits_{k\in \mathcal{P}_n}f_k^o(y_k)\\
    &=\sum\limits_{y_{k\in \mathcal{P}_n}} P(y_n|y_{k\in \mathcal{P}_n}) \prod\limits_{k\in \mathcal{P}_n} P(\mathbf{x}_{tree(k)}, y_k)\\
    &=\sum\limits_{y_{k\in \mathcal{P}_n}}P(y_n|y_{k\in \mathcal{P}_n}) \prod\limits_{k\in \mathcal{P}_n}P(y_k|\mathbf{x}_{tree(k)}) P(\mathbf{x}_{tree(k)})\\
    &=\sum\limits_{y_{k\in \mathcal{P}_n}} P(y_n, y_{k\in \mathcal{P}_n} | \{\mathbf{x}_{tree(k)}, k\in \mathcal{P}_n\}) \prod\limits_{k\in \mathcal{P}_n} P(\mathbf{x}_{tree(k)})\\
    &=\sum\limits_{y_{k\in \mathcal{P}_n}} P(y_n, y_{k\in \mathcal{P}_n} , \{\mathbf{x}_{tree(k)}, k\in \mathcal{P}_n\})\\
    &= P(y_n, \{\mathbf{x}_{tree(k)}, k\in \mathcal{P}_n\})\\
    &= P(y_n, \mathbf{x}_{subtree(n)})
    \end{split}
\end{equation}

\begin{equation}\small
    \begin{split}
     f_n^o(y_n) &= f_n^i(y_n)P(\mathbf{x}_n|y_n)\\
     &=P(y_n, \mathbf{x}_{subtree(n)}) P(\mathbf{x}_n|y_n)\\
     &=P(\mathbf{x}_{subtree(n)} | y_n) P(y_n) P(\mathbf{x}_n|y_n)\\
     &=P(\mathbf{x}_{tree(n)} | y_n) P(y_n)\\
     &=P(\mathbf{x}_{tree(n)} , y_n)\\
     \end{split}
\end{equation}

Thus, Equation \ref{eq:fi} and \ref{eq:fo} hold for internal nodes. Hence Equation \ref{eq:fi} and \ref{eq:fo} are correct for all nodes.

Backward message propagation, the statistical meanings of the message are as below
\begin{align}
    g_n^i(y_n) &= P(\mathbf{x}_{pre(n)} | y_n)\label{eq:gi}\\
    g_n^o(y_n) &= P(\mathbf{x}_{passed(n)} | y_n)\label{eq:go}
\end{align}
Here, $pre(n)$ means all nodes visited before $n$ during top-down propagation in the reverse tree (i.e., $pre(n)$ is all nodes in the reverse tree excluding $tree(n)$), $passed(n)$ means all nodes visited including $n$ during top-down propagation in the reverse tree (i.e., $passed(n)$ includes all nodes in the reverse tree excluding $subtree(n)$).

Base case: If $n$ is root node, 
\begin{align}
    g_n^i(y_n) &= 1\\
    \begin{split}
            g_n^o(y_n) &= g_n^i(y_n)P(\mathbf{x}_n|y_n)\\
                        &= P(\mathbf{x}_n | y_n)
    \end{split}
\end{align}

Thus, the statement Equation \ref{eq:gi} and \ref{eq:go} are correct.

Induction step: Let $y_k$ be an internal node and and suppose Equation \ref{eq:gi} and \ref{eq:go} is true for $n=k$. Then, based on topological order, let $y_n$ be one of the parents of node $y_k$.
\begin{equation}\small
    \begin{split}
    g_n^i(y_n) &=\sum\limits_{y_{c_n},y_{k\in \mathcal{S}_n}} g_{c_n}^o P(y_{c_n}|y_n,y_{k\in \mathcal{S}_n}) \prod\limits_{k\in \mathcal{S}_n}f_k^o(y_k)\\
    &=\sum\limits_{y_{c_n},y_{k\in \mathcal{S}_n}} P(\mathbf{x}_{passed(c_n)} | y_{c_n}) P(y_{c_n}|y_n,y_{k\in \mathcal{S}_n}) \prod\limits_{k\in \mathcal{S}_n} P(\mathbf{x}_{tree(k)} , y_k)\\
    &=\sum\limits_{y_{k\in \mathcal{S}_n}} P(\mathbf{x}_{passed(c_n)} | y_n,y_{k\in \mathcal{S}_n}) \prod\limits_{k\in \mathcal{S}_n} P(\mathbf{x}_{tree(k)} | y_k) P(y_k)\\
    &=\sum\limits_{y_{k\in \mathcal{S}_n}} P(\mathbf{x}_{passed(c_n)} | y_n,y_{k\in \mathcal{S}_n}) \prod\limits_{k\in \mathcal{S}_n} P(\mathbf{x}_{tree(k)} | y_n,y_{k\in \mathcal{S}_n}) P(y_k)\\
    &=\sum\limits_{y_{k\in \mathcal{S}_n}} P(\mathbf{x}_{passed(c_n)}, \{ \mathbf{x}_{tree(k)}, k\in \mathcal{S}_n\} | y_n,y_{k\in \mathcal{S}_n} ) \prod\limits_{k\in \mathcal{S}_n} P(y_k)\\
    &=\sum\limits_{y_{k\in \mathcal{S}_n}} P(\mathbf{x}_{pre(n)} | y_n,y_{k\in \mathcal{S}_n} ) \prod\limits_{k\in \mathcal{S}_n} P(y_k) P(y_n) / P(y_n) \\
    &=\sum\limits_{y_{k\in \mathcal{S}_n}} P(\mathbf{x}_{pre(n)} , y_n,y_{k\in \mathcal{S}_n} )  / P(y_n)\\
    &=P(\mathbf{x}_{pre(n)} , y_n)  / P(y_n)\\
    &=P(\mathbf{x}_{pre(n)} | y_n)
    \end{split}
\end{equation}

\begin{equation}
    \begin{split}
    g_n^o(y_n) &= g_n^i(y_n)P(\mathbf{x}_n|y_n)\\
    &=P(\mathbf{x}_{pre(n)} | y_n)P(\mathbf{x}_n|y_n)\\
    &=P(\mathbf{x}_{passed(n)} | y_n)
    \end{split}
\end{equation}

Thus, Equation \ref{eq:gi} and \ref{eq:go} hold for internal nodes. Hence Equation \ref{eq:gi} and \ref{eq:go} are correct for all nodes.

For unnormalized marginal posterior distribution, Equation \ref{eq:unnormalizedMarginY} and \ref{eq:normalizedMarginYYp} are proved as below. For simplicity, we omit $\boldsymbol{\Theta_0}$ on the right side of the following equations.

\begin{equation}\small
\begin{split}
    P^\prime(y_n|\mathbf{X},\boldsymbol{\Theta_0}) &= f_n^i(y_n) g_n^i(y_n) P(\mathbf{x}_n|y_n) \\
    &= P(\mathbf{x}_{subtree(n)}, y_n) P(\mathbf{x}_{pre(n)} | y_n)P(\mathbf{x}_n|y_n)\\
    &= P(\mathbf{x}_{subtree(n)}| y_n) P(y_n) P(\mathbf{x}_{pre(n)} | y_n)P(\mathbf{x}_n|y_n)\\
    &= P(\mathbf{X}| y_n) P(y_n)\\
    &= P(\mathbf{X}, y_n)
\end{split}
\end{equation}

\begin{equation}
\begin{split}
P(y_n|\mathbf{X},\boldsymbol{\Theta_0}) &= \frac{P^\prime(y_n|\mathbf{X},\boldsymbol{\Theta_0})}{\sum\limits_{y_n}P^\prime(y_n|\mathbf{X},\boldsymbol{\Theta_0})}  \\
&= \frac{P(\mathbf{X}, y_n)}{P(\mathbf{X})}\\
&= P( y_n | \mathbf{X})
\end{split}
\end{equation}

\begin{equation}\small
\begin{split}
&P^\prime(y_n, y_{k\in \mathcal{P}_n}|\mathbf{X},\boldsymbol{\Theta_0})\\ 
&= \prod\limits_{k\in \mathcal{P}_n}f_k^o(y_k) g_n^o(y_n) P(y_n | y_{k\in \mathcal{P}_n})\\
&= \prod\limits_{k\in \mathcal{P}_n} P(\mathbf{x}_{tree(k)} , y_k) P(\mathbf{x}_{passed(n)} | y_n) P(y_n | y_{k\in \mathcal{P}_n})\\
&= P(\mathbf{x}_{subtree(n)} , y_{k\in \mathcal{P}_n}) P(\mathbf{x}_{passed(n)} | y_n) P(y_n | y_{k\in \mathcal{P}_n})\\
&= P(\mathbf{x}_{subtree(n)} | y_{k\in \mathcal{P}_n}) P(y_{k\in \mathcal{P}_n}) P(\mathbf{x}_{passed(n)}, y_n | y_{k\in \mathcal{P}_n})\\
&= P(\mathbf{X}, y_n | y_{k\in \mathcal{P}_n}) P(y_{k\in \mathcal{P}_n})\\
&= P(\mathbf{X}, y_n , y_{k\in \mathcal{P}_n})
\end{split}
\end{equation}

\begin{equation}\small
\begin{split}
P(y_n, y_{k\in \mathcal{P}_n}|\mathbf{X},\boldsymbol{\Theta_0})
&=\frac{P^\prime(y_n, y_{k\in \mathcal{P}_n}|\mathbf{X},\boldsymbol{\Theta_0})}{\sum\limits_{y_n,y_{k\in \mathcal{P}_n}}P^\prime(y_n, y_{k\in \mathcal{P}_n}|\mathbf{X},\boldsymbol{\Theta_0})}    \\
&=\frac{P(\mathbf{X}, y_n , y_{k\in \mathcal{P}_n})}{P(\mathbf{X})} \\
&= P(y_n , y_{k\in \mathcal{P}_n} | \mathbf{X})
\end{split}
\end{equation}

\end{proof}



%




\end{document}